\documentclass[10pt,journal]{IEEEtran}
\usepackage{cite,footnote,xspace,syntonly}
\usepackage{amsmath,mathtools}

\usepackage{cite}
\usepackage{fixltx2e}

\usepackage{graphicx}

\usepackage{subcaption}              %subfigures with \mbox and \subfigure[]

\usepackage[utf8]{inputenc}

\usepackage{amsfonts}
\usepackage{amsmath}
\usepackage{mathtools}
\usepackage{amssymb}

\usepackage{bm}

\usepackage{color,verbatim}
\usepackage{multirow}
\usepackage{accents}

\usepackage{theoremref}
\usepackage{flushend}

\usepackage{url}

\newcounter{rulecounter}
\newcommand{\resetrule}{ \setcounter{rulecounter}{0}}
\resetrule

\newsavebox{\selvestebox}
\newenvironment{colbox}[1]
  {\newcommand\colboxcolor{#1}%
   \begin{lrbox}{\selvestebox}%
   \begin{minipage}{\dimexpr\columnwidth-2\fboxsep\relax}}
  {\end{minipage}\end{lrbox}%
   \begin{center}
   \colorbox{\colboxcolor}{\usebox{\selvestebox}}
   \end{center}}

\definecolor{orange}{rgb}{1,0.8,0}
\definecolor{gray}{rgb}{.9,0.9,0.9}
\definecolor{darkgray}{rgb}{.3,0.3,0.3}
\definecolor{darkblue}{rgb}{.1,0.0,0.3}
\definecolor{lightblue}{rgb}{0.7,0.7,1}
\definecolor{lightred}{rgb}{1,0.7,.7}
\definecolor{purple}{RGB}{204,153,255}
\definecolor{lightgray}{rgb}{.95,0.95,0.95}
\definecolor{lightgreen}{rgb}{0.3,0.5,0.3}
\definecolor{darkgreen}{rgb}{0.05,0.3,0.05}

\newcommand{\ra}{$\rightarrow$~}

\newcommand{\inv}{^{-1}}

\newcommand{\rfield}{\mathbb{R}}

\newcommand{\diag}[1]{\mathop{\rm diag}({#1})}

\newcommand{\tr}[1]{\mathop{\rm tr}\left(#1\right)}

 \newcommand{\define}{:=}
\newcommand{\minimize}{\mathop{\text{minimize}}}

\newcommand{\st}{\mathop{\text{s.t.}}}

\newtheorem{myproposition}{Proposition}
\newtheorem{myremark}{Remark}
\newtheorem{myproblemstatement}{Problem Statement}
\newtheorem{mylemma}{Lemma}
\newtheorem{mytheorem}{Theorem}

\newtheorem{mycorollary}{Corollary}

\usepackage{pgfplots}
\pgfplotsset{compat=newest}
\usetikzlibrary{plotmarks}
\usetikzlibrary{arrows.meta}
\usepgfplotslibrary{patchplots}
\usepackage{grffile}

\pgfplotsset{plot coordinates/math parser=false}
\newlength\mywidth
\newlength\myheight
\definecolor{mycolorBL1}{rgb}{0.00000,0.54700,0.54100}%
\definecolor{mycolorBL2}{rgb}{0.85000,0.92500,0.09800}%
\definecolor{mycolorBL3}{rgb}{0.92900,0.69400,0.72500}%
\definecolor{mycolor1}{rgb}{0.00000,0.44700,0.74100}%
\definecolor{mycolor2}{rgb}{0.85000,0.32500,0.09800}%
\definecolor{mycolor3}{rgb}{0.92900,0.69400,0.12500}%
\definecolor{mycolor4}{rgb}{0.89400,0.18400,0.15600}%
\definecolor{mycolor5}{rgb}{0.46600,0.67400,0.18800}%
\definecolor{mycolor6}{rgb}{0.30100,0.74500,0.93300}%
\definecolor{mycolor7}{rgb}{0.63500,0.07800,0.18400}%
\definecolor{colorJSIGoT}{rgb}{1,0.0032500,0.001}%
\newcommand{\gditerind}{\tau}
	\newcommand{\gditernot}[1]{^{(#1)}}
\newcommand{\gdstep}{\hc{\theta}}

\newcommand{\timeind}{\hc{t}}
\newcommand{\regsmoothadj}{\hc{\mu}_A}
\newcommand{\timepluslagnum}{\hc{t_w}}

\newcommand{\timenot}[1]{_{#1}}
\newcommand{\lagnot}[1]{^{(#1)}}
\newcommand{\timenum}{\hc{T}}
\newcommand{\timegiventimenot}[2]{_{{#1}|{#2}}}
\newcommand{\errorcovkf}{\hc{\bm \Sigma}}
\newcommand{\lagnum}{\hc{w}}
\newcommand{\obsrowset}{\hc{\mathcal{R}}}
\newcommand{\obscolumset}{\hc{\mathcal{C}}}
\newcommand{\signalfun}{{\hc{y}}}
\newcommand{\signalvec}{\hc{\mathbf{ \signalfun}}}
\newcommand{\signalfulvec}{\hc{\bm{\psi}}} 
\newcommand{\signalmat}{\hc{\mathbf{ Y}}}

\newcommand{\adjacencymatrow}{\hc{\underline{\mathbf{a}}}}
\newcommand{\adjacencymatrowother}{\hc{\underline{\mathbf{b}}}}

\newcommand{\signalestvec}{\hc{\hat{\mathbf{\signalfun}}}} 
\newcommand{\observationnoisefun}{\hc{\epsilon}}
\newcommand{\observationnoisevec}{\bm{\observationnoisefun}}
\newcommand{\lossfunction}{\mathcal{F}}
\newcommand{\observationfun}{{\hc{z}}} 
\newcommand{\observationvec}{\hc{\mathbf{\observationfun}}} 
\newcommand{\observationvecwithmis}{\hc{\tilde{\mathbf{\observationfun}}}} 
\newcommand{\observationfunwithmis}{\hc{\tilde{\observationfun}}}

\newcommand{\graph}{\hc{\mathcal{G}}}

\newcommand{\vertexset}{\hc{\mathcal{V}}}

\newcommand{\noiseadjacencymat}{\hc{\mathbf{{E}}}}
\newcommand{\noiseadjacencymatentry}{\hc{{E}}}
\newcommand{\identitymat}{\hc{\mathbf{ I}}}

\newcommand{\adjacencymat}{\hc{\mathbf{ A}} }
\newcommand{\adjacencyzeromat}{\hc{\mathbf{ A}}^{(0)} }
\newcommand{\adjacencyzeromatentry}{\hc{ a}^{(0)} }
\newcommand{\adjacencyonematentry}{\hc{ a}^{(1)}}
\newcommand{\adjacencyonemat}{\hc{\mathbf{ A}}^{(1)}}
\newcommand{\adjacencyestmat}{\hc{\hat{\mathbf{ A}}}} 
\newcommand{\adjacencyestmatentry}{\hc{\hat{a}}}

\newcommand{\adjacencymatentry}{\hc{ a}} 
\newcommand{\adjacencyset}{\hc{\mathcal{A} }} 
\newcommand{\layeridex}{\hc{t}}

\newcommand{\layernot}[1]{{_{#1}}}  % 
\newcommand{\layernum}{{\hc{T}}}
\newcommand{\laplacianevec}{\hc{\mathbf{ u}} }
\newcommand{\vertexind}{{\hc{{n}}}}
\newcommand{\vertexindp}{{\hc{{\vertexind}'}}} % vertex index prime

\newcommand{\vertexnum}{{\hc{{N}}}}

\newcommand{\laplacianmat}{\hc{\mathbf{ L}} }

\newcommand{\sampleset}{\hc{\mathcal{M}}} 
\newcommand{\samplemat}{\hc{\mathbf{ M}} }

\newcommand{\samplematwithmis}{\hc{\tilde{\mathbf{ M}}} }
\newcommand{\samplematwithmisentry}[3]{\hc{\tilde{ m}}^{(#3)}_{#1,#2} }

\newcommand{\sampleind}{{\hc{m}}} 
\newcommand{\samplenum}{{\hc{M}}} 
\newcommand{\transpose}{^{\hc{\top}}}
\newcommand{\regpar}{\hc{\mu}}

\newcommand{\regparadjone}{\hc{\lambda}_1}
\newcommand{\regparadjtwo}{\hc{\lambda}_2}

\newcommand{\regfun}{\hc{\rho}}
\newcommand{\regfunelnet}{\hc{\rho_{e}}}
\newcommand{\vertexvertexnot}[2]{_{#1,#2}}  %
\newcommand{\vertexnot}[1]{_{#1}}  %
\newcommand{\seedzeromat}{\hc{\mathbf{ D}_0}}

\newcommand{\seedmat}{\hc{\mathbf{ D}}}
\newcommand{\seedmatentry}{\hc{ D}}
\newcommand{\bandwidth}{\hc{B}}

\newcommand{\vertlayernot}[2]{_{{#1}{#2}}}
\newcommand{\vertimenot}[2]{_{{#1}{#2}}}

\newcommand{\semnoisefun}{\hc{\eta}}
\newcommand{\semnoisevec}{\bm{{\semnoisefun}}}

\newcommand{\singinfeonefun}{\hc{g}}

\newcommand{\iterateindex}{{\hc{i}}}
\newcommand{\iteratenot}[1]{{\hc{(}{#1}\hc{)}}}

\newcommand{\zerostatevec}{\hc{\bm\mu}\timenot{0}}

\newcommand{\transmat}{\hc{\mathbf{F}}}
\newcommand{\transfulmat}{\hc{\bm{\Phi}}}
\newcommand{\transcov}{\hc{\mathbf{Q}}}
\newcommand{\corfulmat}{\hc{\bm{\Omega}}}
\newcommand{\auxadj}{\tilde{\adjacencymat}}
\newcommand{\coverrror}{\hc{\mathbf{\Sigma}}}
\newcommand{\kalmansmoothgain}{\hc{\mathbf{G}}^s}
\newcommand{\kalmanfiltgain}{\hc{\mathbf{G}}^f}

\newcommand{\signalmissmat}{\signalmat^{\Omega}}

\newcommand{\observationmatwithmis}{\hc{\tilde{\mathbf{Z}}}}

\newcommand{\observationfulvec}{\hc{\bm{\zeta}}}
\newcommand{\observationmatwithmisrow}{\hc{\tilde{\underline{\mathbf{z}}}}}
\newcommand{\extadjacencymat}{\hc{\bar{\mathbf{A}}}}
\newcommand{\extsignalvec}{\hc{\bar{\mathbf{\signalfun}}}}
\newcommand{\adjacencymathelp}{\hc{\bm{\Psi}}}
\newcommand{\admmhelp}{\hc{\boldsymbol{U}}}
\newcommand{\admmreg}{\hc{\rho}}
\newcommand{\cormat}{\hc{\boldsymbol{R}}}
\newcommand{\crosscormat}{\hc{\boldsymbol{C}}}
\newcommand{\admmnot}[1]{\hc{[}{#1}\hc{]}}
\newcommand{\admmiter}{i}
\newcommand{\sparsitynum}{\hc{S}}
\usepackage{float}
\usepackage{algorithmic}
\usepackage[ruled,vlined,resetcount]{algorithm2e}
\usepackage[normalem]{ulem}
\usepackage{amsthm}
\usepackage{tabularx}
\usepackage{epstopdf}
\theoremstyle{plain}

\newtheorem{mytheoremhere}{Theorem}
\newtheorem{mypropositionhere}{Proposition}
\newtheorem{mydefinition}{Definition}
\theoremstyle{definition}
\newtheorem{myremarkhere}{Remark}

\newcommand{\cmt}[1]{} % comment
\newcommand{\hc}[1]{\textcolor{black}{#1}} % highlight command --> to
\newcolumntype{Y}{>{\centering\arraybackslash}X}
\newif\ifshowtikz
\showtikztrue
\allowdisplaybreaks[4]
\let\oldtikzpicture\tikzpicture
\let\oldendtikzpicture\endtikzpicture

\renewenvironment{tikzpicture}{%
\ifshowtikz\expandafter\oldtikzpicture%
\else\comment%   
\fi
}{%
\ifshowtikz\oldendtikzpicture%
\else\endcomment%
\fi
}

\makeatletter
\renewcommand*\env@matrix[1][*\c@MaxMatrixCols c]{%
\hskip -\arraycolsep
\let\@ifnextchar\new@ifnextchar
\array{#1}}
\makeatother
\begin{document}

\title{Semi-Blind Inference of Topologies and\\
	Dynamical Processes over Graphs } 
%\name{Vassilis N. Ioannidis, Yanning Shen, and Georgios B. Giannakis
%		\thanks{
%			The work in this paper was supported by 
%			 NSF grant 1500713, and NIH
%			 1R01GM104975-01.}}
%	
%	\address{ECE Dept. and Digital Tech. Center, Univ. of Minnesota, Mpls, MN 
%	55455, USA\\
%		E-mails: \{ioann006, shenx513, georgios\}@umn.edu}
\author{Vassilis N. Ioannidis, Yanning Shen, and Georgios B. 
Giannakis
\thanks{
The work in this paper was supported by 
NSF grants  171141,  1500713, and  1442686.}
\\ECE Dept. and Digital Tech. Center, Univ. of Minnesota, Mpls, MN 55455, USA\\
E-mails: $\{$ioann006, shenx513, georgios$\}$ @umn.edu
}

\maketitle

\begin{abstract}
Network science provides valuable insights across  numerous disciplines including sociology, biology, neuroscience and engineering. A task of major practical importance in these application domains is inferring the network	structure from noisy observations at a subset of nodes. Available methods for topology inference typically assume that the  process over the network is  observed at all nodes. However, application-specific constraints may prevent acquiring network-wide observations.  Alleviating the limited flexibility	of existing approaches, this work advocates structural models for  graph processes and develops novel algorithms for joint inference of the network topology and processes from partial nodal observations. Structural equation models (SEMs) and structural vector autoregressive models (SVARMs) have well-documented merits in identifying even directed topologies of complex	graphs; while SEMs capture contemporaneous causal dependencies among nodes, SVARMs further account for time-lagged influences. This paper develops algorithms that iterate between inferring 	directed graphs that ``best" fit the data, and estimating the network processes at reduced computational complexity by leveraging tools related to Kalman smoothing.	To further accommodate delay-sensitive applications, an online joint inference approach is put forth that even tracks time-evolving topologies. Furthermore, conditions for identifying the network topology given partial observations are specified. It is proved that the required number of observations for unique identification reduces	significantly when the network structure is sparse. Numerical tests with  synthetic as well as real datasets corroborate the effectiveness of the novel approach.
\end{abstract}

\section{Introduction}
Modeling vertex attributes as  processes that take values over a graph allows for data processing tasks, such as filtering, inference, and compression, while accounting for information captured by the network topology~\cite{shuman2013emerging,kolaczyck2009}.
However, if the  topology is unavailable, inaccurate or even unrelated to the process of interest,  performance of the associated  task may degrade severely. For example, consider a social graph where the goal is to predict the salaries of all individuals given the salaries of some. Graph-based inference approaches that assume smoothness of the salary over the given graph, may fall short if the salary is dissimilar among friends.

Topology identification is possible when observations at all nodes are available by employing structural models, see e.g., \cite{kaplan09}. However, in  many real settings one can only afford to collect nodal observations from a subset of nodes due to application-specific restrictions. For example, sampling all nodes may be prohibitive in massive graphs; in social networks individuals may be reluctant to share personal information due to privacy concerns; in sensor networks, devices may report  measurements sporadically to save energy; and in gene regulatory networks,  gene expression data may contain misses due  to experimental errors. In this context, the present paper relies on SEMs~\cite{kaplan09}, and SVARMs~\cite{chen2011vector} and aims at  \emph{jointly} inferring  the network topology and estimating graph signals, given noisy observations at subsets of nodes.

SEMs provide a statistical  framework for inference of causal relationships among  nodes~\cite{kaplan09,giannakis2017tutor}. Linear SEMs have been widely adopted in fields as diverse as sociometrics~\cite{goldberger1972structural}, psychometrics~\cite{muthen1984general}, recommender systems~\cite{ning2011slim}, and genetics~\cite{cai2013inference}. Conditions for identifying the network topology under the SEM  have been also provided~\cite{bazer2013ident,shen2017tensor}, but require observations of the process at all nodes.  Recently, nonlinear SEMs  have been developed  to also capture nonlinear 
interactions~\cite{shen2017tsp}. On the 
other hand, SVARMs postulate that nodes further exert time-lagged dependencies 
on one another, and are appropriate for modeling multivariate time 
series~\cite{chen2011vector}. Nonlinear SVARMs have been employed to identify 
directed dependencies between regions of interest  in the 
brain~\cite{shen2016nonlineartmi}. Other approaches identify undirected 
topologies provided that the  graph signals are smooth over the 
graph~\cite{dong2015learning}; or, that the observed process is graph-bandlimited~\cite{sardel16bandl}.
%; or, that the observations are generated by a 
%\textsl{}diffusion process~\cite{thanou2017learning}.  
All these contemporary approaches assume that samples of the graph process are available over all nodes. However, acquiring network-wide observations may incur 
prohibitive sampling costs, especially for massive networks. %\acom{discuss 
%about sardeliti et al at globalsip 16}

Methods for inference of graph signals (or processes), typically assume that 
the network topology is \emph{known} and \emph{undirected}, and the graph 
signal is \emph{smooth}, in the sense that neighboring vertices have
similar  values~\cite{smola2003kernels}.  Parametric 
approaches  adopt the graph-bandlimited 
model~\cite{anis2016proxies,narang2013localized}, which 
postulate that the signal  lies in a graph-related $\bandwidth$-dimensional 
subspace; see  \cite{lorenzo2016lms} for time-varying signals. Nonparametric 
techniques  employ kernels on graphs for inference~\cite 
{smola2003kernels,romero2016spacetimekernel}; see also \cite{ioannidis2016semipar} 
for semi-parametric alternatives. Online data-adaptive  algorithms for 
reconstruction of dynamic processes over dynamic graphs have been proposed 
in~\cite{ioannidis2017kriged}, where kernel dictionaries are generated from the 
network topology. However, performance of the aforementioned techniques may 
degrade when the process of interest is not smooth over the adopted graph.	

To recapitulate, existing approaches either infer the graph process given the 
known topology and nodal observations, or estimate the network topology given 
the process values over all the nodes. The present paper fills this gap by 
introducing algorithms based on SEMs and SVARMs for joint inference of network 
topologies and  graph processes over the underlying graph. The approach is 
\emph{semi-blind}  because it performs the joint estimation task with only 
\emph{partial} observations over the network nodes.  Specifically, the 
contribution is threefold.% \acom{novel formulation}
%, while the flexible formulation allows for incorporating potential prior 
%information about the network topology
\begin{itemize}	
\item[C1.]  A novel  approach is proposed for \emph{joint} inference of 
\emph{directed} network topologies and  signals over the underlying 
graph using SEMs. An efficient algorithm is developed with provable 
convergence at least to a stationary point. 

\item[C2.] To further accommodate temporal dynamics, we advocate  a 
SVARM to infer dynamic processes and graphs. A batch solver is provided that 
alternates between topology estimation and signal inference  with linear  complexity across time. Furthermore, a novel online algorithm is developed that 
performs real-time joint estimation, and tracks time-evolving topologies.

\item[C3.] Analysis of the partially observed noiseless SEM is  
provided that establishes sufficient conditions for identifiability of 
the unknown topology. These conditions suggest that the required 
number of  observations for identification reduces	
significantly when the network exhibits edge sparsity.

\end{itemize}

The rest of the paper is organized as follows. Sec. \ref{sec:problform} reviews 
the SEM and SVARM, and states the problem. Sec. 
\ref{sec:jointinferalgo} presents a novel estimator for joint inference based 
on SEMs.  Sec. \ref{sec:svar} develops both batch and online algorithms for 
inferring dynamic processes and  networks using SVARMs. 
Sec.~\ref{sec:identanal} presents the identifiability results of the partially 
observed SEM.
Finally, numerical experiments and conclusions are presented in 
Secs.~\ref{sec:sims} and~\ref{sec:concl}, respectively. 

\emph{Notation:}
Scalars are denoted by
lowercase, column vectors by bold lowercase, and matrices 
by bold
uppercase letters. Superscripts $~\transpose$ and $~\inv$
respectively 
denote transpose and inverse; while $\boldsymbol 1_N$ stands for the $N\times1$ all-one 
vector. Moreover, 
$[\mathbf{A}]_{i,j}$ denotes a block entry of appropriate size. 
Finally, if $\mathbf A$ 
is a matrix and $\mathbf x$ a vector, then $||\mathbf x||^2_{\mathbf
A}:= \mathbf x\transpose \mathbf A\inv \mathbf x$,  $||\mathbf x
||_2^2:= \mathbf x\transpose \mathbf x$, $\|\mathbf A\|_1$ denotes 
the 
$\mathcal{L}_1$-norm of the 
vectorized matrix, and $\|\mathbf A\|_F^2$ is the Frobenius norm of $ 
\mathbf A$.
%			 $(\boldsymbol A)_{m,n}$ is the  $(m,n)$-th
%			entry of matrix $\boldsymbol A$.
%			   The $n$-th column of the identity matrix
%			$\identitymat$ is represented by $\canonicalvec{n}$.  
%
%		     ||\boldsymbol x||_{\identitymat}$.    $\pdset^N$
%			represents the cone of $N\times N$ positive definite 
%matrices.
%			 Finally, $\delta[\cdot]$ stands for the Kronecker 
%delta,
%			% $\indicator[C]$ is the indicator of condition $C$, 
%%taking 
%			%the  value 1 if the condition is satisfied and 0 otherwise
%			 and $\expectednb$ for  expectation.

\section{Structural models and problem formulation}
%\subsection{Static structural equation model}
\label{sec:problform}
Consider a network with $\vertexnum$ nodes modeled by the 
graph  $\graph:=(\vertexset,\adjacencymat)$, where 
$\vertexset:=\{v_1, \ldots, v_\vertexnum\}$ is the set of vertices and 
$\adjacencymat$ denotes the $\vertexnum\times \vertexnum$ adjacency matrix, 
whose $(\vertexind,\vertexindp)$-th  entry 
$\adjacencymatentry\vertexvertexnot{\vertexind}{\vertexindp}$ represents the 
weight of the directed edge from $v_\vertexindp$ to $v_\vertexind$. 
\cmt{linear SEM}A real-valued process (or 
signal) on $\graph$ is a map $\signalfun\layernot{\layeridex} : \mathcal{V} 
\rightarrow
\mathbb{R}$.
In social networks (e.g., 
Twitter) over which information diffuses $\signalfun\vertlayernot{\vertexind}{\layeridex}$ could represent the 
timestamp when 
subscriber $n$ tweeted about a viral story $t$.	Since real-world networks often 
exhibit edge sparsity, 	$\adjacencymat$ has only a few nonzero entries. 
\subsection{Structural models}
%\noindent\textbf{Structural equation model.} 
The linear 
SEM\cite{goldberger1972structural} postulates that  
$\signalfun\vertlayernot{\vertexind}{\layeridex}$ depends linearly 
on 
$\{\signalfun\vertlayernot{\vertexindp}{\layeridex}
\}_{\vertexindp\ne\vertexind}$,
that amounts to
\begin{align}
\label{eq:semscalar}
\signalfun\vertlayernot{\vertexind}{\layeridex}= 
\sum_{\vertexind\ne\vertexindp} 
\adjacencymatentry\vertexvertexnot{\vertexind} 
{\vertexindp}\signalfun\vertlayernot{\vertexindp}{\layeridex} 
+\semnoisefun\vertlayernot{\vertexind}{\layeridex}
\end{align}
where the unknown $	\adjacencymatentry\vertexvertexnot{\vertexind} 
{\vertexindp}$ captures the causal influence of node 	$v_{\vertexindp}$ upon 
node  
$v_{\vertexind}$,  and 
$\semnoisefun\vertlayernot{\vertexind}{\layeridex}$ accounts for unmodeled 
dynamics.  Clearly, $\eqref{eq:semscalar}$  suggests that 
$\signalfun\vertlayernot{\vertexind}{\layeridex}$ is influenced  directly by nodes in its neighborhood $\mathcal{N}_\vertexind:=\{v_\vertexindp: 
\adjacencymatentry\vertexvertexnot{\vertexind} 
{\vertexindp}\ne0\}$.
With the $\vertexnum\times1$ 
vectors 
$\signalvec\layernot{\layeridex}:=[\signalfun\vertlayernot{1}{\layeridex},
\ldots,\signalfun\vertlayernot{\vertexnum}{\layeridex}]\transpose$, 
and 
$\semnoisevec\layernot{\layeridex}:=[\semnoisefun\vertlayernot{1}{\layeridex},
\ldots,\semnoisefun\vertlayernot{\vertexnum}{\layeridex}]\transpose$, 
\eqref{eq:semscalar} can 
be written 
in matrix-vector form as
\begin{align}
\label{eq:semvec}
\signalvec\layernot{\layeridex} =\adjacencymat\signalvec\layernot{\layeridex}
%+\exogenousweightmat\exogenousvec\layernot{\layeridex} 
+\semnoisevec\layernot{\layeridex},~\layeridex=1,\ldots,\layernum.
\end{align}
SEMs  have been successful in a host of applications, including
gene regulatory networks \cite{cai2013inference}, and recommender systems
\cite{ning2011slim}. Therefore, the index $\layeridex$ does not necessarily 
indicate time, but may represent different  individuals (gene regulatory 
networks), or  movies (recommender systems). 
An interesting consequence emerges if one considers 	
$\semnoisevec\layernot{\layeridex}$ as a random process with 
$\semnoisevec\layernot{\layeridex} 
\sim\mathcal{N}(\boldsymbol{0},\sigma^2_{\layeridex}\identitymat_\vertexnum)$. 
Thus,  \eqref{eq:semvec} can be written as 	
$\signalvec\layernot{\layeridex}=(\identitymat_\vertexnum- \adjacencymat)\inv 
\semnoisevec\layernot{\layeridex}=\sum_{n=0}^{\infty}\adjacencymat^n 
\semnoisevec\layernot{\layeridex}$ with 	
$\signalvec\layernot{\layeridex}\sim\mathcal{N}(\boldsymbol{0},	
\sigma^2_{\layeridex}\sum_n\adjacencymat^n)$ having covariance matrix 
$\mathbf{C}_{\signalvec\layernot{\layeridex}}:=
%\expectednb[\signalvec\layernot{\layeridex} 
%\signalvec\layernot{\layeridex}\transpose] =
\sigma^2_{\layeridex}\sum_n\adjacencymat^n$. Matrices
$\mathbf{C}_{\signalvec\layernot{\layeridex}}$ and $\adjacencymat$ are 
simultaneously diagonalizable, and hence  $\signalvec\layernot{\layeridex}$  
is  a graph stationary process 
\cite{marques2016stationary}.

\begin{figure}[t]
\centering
\includegraphics[width=8.cm]{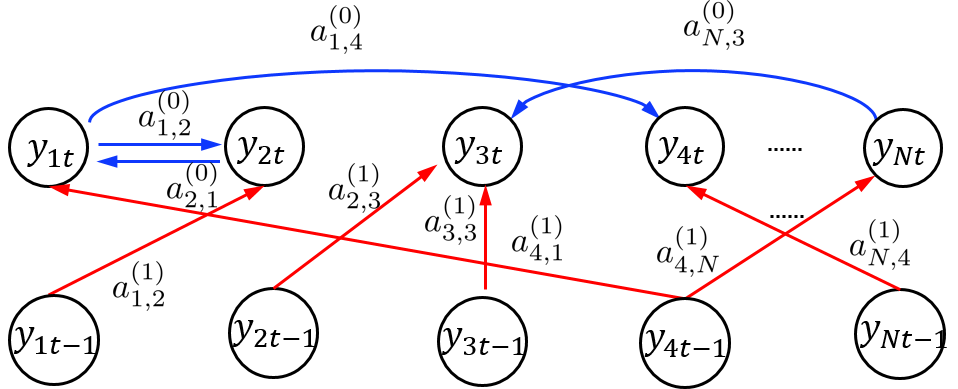}
\caption{The SVARM \eqref{eq:svarmvec}; 
instantaneous dependencies 
$\{\adjacencyzeromatentry\vertexvertexnot{\vertexind}{\vertexindp}\}$ (blue 
arrows) and time-lagged influences 
$\{\adjacencyonematentry\vertexvertexnot{\vertexind}{\vertexindp}\}$ (red 
arrows). }\label{fig:svarm}
\end{figure}
In order to unveil the hidden causal network topology,  SVARMs 
postulate that each $\signalfun\vertimenot{\vertexind}{\timeind}$
can be represented as a linear combination of instantaneous measurements at 
other nodes 
$\{\signalfun\vertimenot{\vertexindp}{\timeind}\}_{\vertexindp\ne\vertexind}$, 
and their time-lagged versions   
$\{\signalfun\vertimenot{\vertexindp}{\timeind-1}\}_{
\vertexindp=1}^\vertexnum$~\cite{chen2011vector}.
Specifically, the following  instantaneous plus  time-lagged model is advocated
\begin{align}
\label{eq:svarmscalar}
\signalfun\vertimenot{\vertexind}{\timeind}=
\sum_{\vertexindp\ne\vertexind}
\adjacencyzeromatentry\vertexvertexnot{\vertexind}{\vertexindp}
\signalfun\vertimenot{\vertexindp}{\timeind}
+	
\sum_{\vertexindp=1}^\vertexnum\adjacencyonematentry\vertexvertexnot{\vertexind}
{\vertexindp}\signalfun\vertimenot{\vertexindp}{(\timeind-1)}+
\semnoisefun\vertimenot{\vertexind}{\timeind}
\end{align}
where $	\adjacencyzeromatentry\vertexvertexnot{\vertexind}{\vertexindp}$ 
captures the instantaneous 
causal influence of node $\vertexind$ upon node $\vertexindp$, 
$\adjacencyonematentry\vertexvertexnot{\vertexind}{\vertexindp}$ encodes the 
time-lagged causal influence between them, and 
$\semnoisefun\vertimenot{\vertexind}{\timeind}$ accounts for unmodeled 
dynamics. By defining 
$\signalvec\layernot{\timeind}:=[\signalfun\vertlayernot{1}{\timeind},
\ldots,\signalfun\vertlayernot{\vertexnum}{\timeind}]\transpose$, 
$\semnoisevec\layernot{\timeind}:=[\semnoisefun\vertlayernot{1}{\timeind},
\ldots,\semnoisefun\vertlayernot{\vertexnum}{\timeind}]\transpose$, and the 
$\vertexnum\times \vertexnum$ matrices $\adjacencyzeromat$, and 
$\adjacencyonemat$ with entries 
$\{\adjacencyzeromatentry\vertexvertexnot{\vertexind}{\vertexindp}\}$, and 
$\{\adjacencyonematentry\vertexvertexnot{\vertexind}{\vertexindp}\}$ 
respectively,  the matrix-vector form of \eqref{eq:svarmscalar} becomes
\begin{align}
\label{eq:svarmvec}
\signalvec\timenot{\timeind}
=\adjacencyzeromat\signalvec\timenot{\timeind}
+\adjacencyonemat\signalvec\timenot{\timeind-1}
+\semnoisevec\timenot{\timeind},~\timeind=1,\ldots,\timenum
\end{align}
with $\signalvec\timenot{0}=\zerostatevec+\semnoisevec\timenot{0}$, and $\zerostatevec$ considered known. The SVARM in \eqref{eq:svarmvec} is a better fit for time-series over graphs compared to the SEM in \eqref{eq:semvec}, because it further accounts for temporal dynamics of $\signalvec\timenot{\timeind}$ through the 	time-lagged influence term $\adjacencyonemat\signalvec\timenot{\timeind-1}$. For this reason, SVARMs will be employed for  dynamic setups, such as modeling ECoG time series in brain networks, and predicting Internet router delays. The SVARM is depicted in Fig.~\ref{fig:svarm}.

\subsection{Problem statement}

%\noindent\textbf{Observation model.} 
Application-specific  constraints  allow only for a limited number of samples 
across nodes per slot $\timeind$. Suppose that ${\samplenum\layernot{\layeridex}}$ noisy samples of the $\layeridex$-th observation vector
\begin{align}
\label{eq:obsscalar}
\observationfun\vertlayernot{\sampleind}{\layeridex}= 
\signalfun\vertlayernot{\vertexind_\sampleind}{\layeridex} 
+\observationnoisefun\vertlayernot{\sampleind}{\layeridex} 
,~ \sampleind=1,\ldots, \samplenum\layernot{\layeridex}
\end{align}
are available,
where $\sampleset\layernot{\layeridex}:= 
\{\vertexind_1,\ldots,\vertexind_{\samplenum\layernot{\layeridex}}\}$ contains 
the indices  
$1\le\vertexind_1\le\ldots\le\vertexind_{\samplenum\layernot{\layeridex}}\le\vertexnum$
of the 
sampled vertices, and
$\observationnoisefun\vertlayernot{\sampleind}{\layeridex}$ models the 
observation error. With
$\observationvec 
\layernot{\layeridex}:=[\observationfun\vertlayernot{1}{\layeridex},\ldots, 
\observationfun\vertlayernot{\samplenum\layernot{\layeridex}}{\layeridex}]\transpose$,
and 
$\observationnoisevec\layernot{\layeridex}:= 
[\observationnoisefun\vertlayernot{1}{\layeridex},\ldots,
\observationnoisefun\vertlayernot{\samplenum\layernot{\layeridex}}{\layeridex}]\transpose$,
the observation model is
\begin{align}
\label{eq:obsvec}
\observationvec\layernot{\timeind}=\samplemat\layernot{\timeind} 
\signalvec\layernot{\timeind}+\observationnoisevec
\layernot{\timeind},~\timeind=1,\ldots,\timenum
\end{align}
where $\samplemat\layernot{\layeridex}$ is an 
$\samplenum\layernot{\layeridex}\times\vertexnum$ 
matrix with entries 
$\{(\sampleind,\vertexind_\sampleind)\}_{
\sampleind=1}^{\samplenum\layernot{\layeridex}}$,
set to 
one, and the rest set to zero.

%\noindent\textbf{Problem statement.} 
The broad goal of this paper is the \emph{joint 
inference} of the hidden network topology and signals over graphs (JISG) from 
partial observations of the latter. Given the observations $\{\observationvec\layernot{\layeridex} \}_{\layeridex=1}^\layernum$ collected in accordance to the sampling matrices $\{\samplemat\layernot{\layeridex}\}_{\layeridex=1}^\layernum$,  one aims at finding the underlying topologies, $\adjacencymat$ for the SEM, or $\adjacencyzeromat$ and $ \adjacencyonemat$ for the SVARM, as well as reconstructing the graph process at all nodes $\{\signalvec\layernot{\layeridex}\}_{\layeridex=1}^\layernum$.
The complexity of the estimators should preferably 
scale linearly in $\timenum$. 
%\end{enumerate}
As estimating the topology and 
$\{\signalvec\layernot{\layeridex}\}_{\layeridex=1}^\layernum$ relies on
partial observations, this is a semi-blind inference task.

\section{\!Jointly inferring topology and signals}
%\label{sec:jointinfer}
%This section....
%\subsection{Joint BCD algorithm}
\label{sec:jointinferalgo}
\cmt{optimization problem}

%\cmt{naive approach}\acom{sos measure unknowns remove this sentence}
%\change{Clearly, solving~\eqref{eq:obsvec} directly for each 
%$\signalvec\layernot{\layeridex}$ is an 
%undetermined problem since $\samplenum\layernot{\layeridex}\le\vertexnum$.
%Likewise, solving \eqref{eq:semvec} for 
%$\signalvec\layernot{\layeridex},\adjacencymat$ is clearly
%undetermined even with the sparsity assumption. }
\cmt{Joint opt}
Given $\{\observationvec\layernot{\layeridex}\}_{\layeridex=1}^\layernum$ in 
\eqref{eq:obsvec}, this section develops  a novel 
approach to infer 
$\adjacencymat$, and 
$\{\signalvec\layernot\layeridex\}_{\layeridex=1}^\layernum$. To 
this end, we advocate the following regularized  
least-squares (LS) optimization problem
\begin{align}
\label{eq:topsiginfer}
\underset{\adjacencymat\in\adjacencyset,\{\signalvec\layernot\layeridex\}_{\layeridex=1}
^\layernum} 
{\min}&{\lossfunction(\adjacencymat,\{\signalvec\layernot\layeridex\}_{\layeridex=1}
^\layernum ):=} 
\sum_{\layeridex=1}^{\layernum}\|\signalvec\layernot{\layeridex}- 
\adjacencymat
\signalvec\layernot{\layeridex}\|_2^2\\&\hspace{-1cm}+
\sum_{\layeridex=1}^{\layernum}\frac{\regpar} 
{\samplenum\layernot{\layeridex}}\|\observationvec\layernot{\layeridex}- 
\samplemat\layernot{\layeridex}\signalvec\layernot{\layeridex}\|_2^2
%+\regparadj\regfun(\adjacencymat)
+\regparadjone\|\adjacencymat\|_1+\regparadjtwo 
\|\adjacencymat\|_F^2\nonumber
%\\
%\st&\adjacencymat\~\adjacencymatentry\vertexvertexnot{\vertexind}{\vertexind}=0,~\vertexind=1,\ldots,\vertexnum
\end{align}
where $\regpar\ge0$ tunes the relative importance of the fitting term; 
%relative 
%\change{to the 
%SEM 
%related terms}, 
$\regparadjone\ge0$, $\regparadjtwo\ge0$ control the effect of the 
$\ell_1$-norm and 
the Frobenius-norm, respectively, and 
$\adjacencyset:=\{\adjacencymat:
\adjacencymat\in\mathbb{R}^{\vertexnum\times\vertexnum}, 
\{\adjacencymatentry\vertexvertexnot 
{\vertexind}{\vertexind}=0\}_{\vertexind=1}^\vertexnum\}$. 
%\remove{Each summand in 
%the fitting term is weighted by 
%$\frac{1}{\samplenum\layernot{\layeridex}}$ to account for the possibly 
%different 
%$|\sampleset\layernot{\layeridex}|$.}
\cmt{elastic net}The weighted sum of $\|\cdot\|_1$ and 
$\|\cdot\|_F$ is 
the so-termed elastic net penalty, which promotes connections between highly 
correlated nodal 
measurements. The elastic net  targets the ``sweet spot'' 
between 
the $\ell_1$ regularizer that effects sparsity, and the 
$\|\cdot\|_F$ regularizer, which advocates fully connected networks~\cite{zou2005regularization}.

Even though \eqref{eq:topsiginfer} is nonconvex in both $\adjacencymat$ and 
$\signalvec\layernot{\layeridex}$ due to the  bilinear product 
$\adjacencymat\signalvec\layernot{\layeridex}$, 
it is convex with respect to (w.r.t.) each  block variable separately.
This  motivates an iterative block coordinate descent (BCD) algorithm that alternates between estimating $\{\signalvec\layernot\layeridex\}_{\layeridex=1}^\layernum$ and  
$\adjacencymat$. 

Given $\adjacencyestmat$, the estimates  
$\{\signalestvec\layernot\layeridex\}_{\layeridex=1}^\layernum$ are found 
by solving the  quadratic problem
\begin{align}
\label{eq:siginfer}
\underset{\{\signalvec\layernot\layeridex\}_{\layeridex=1}^\layernum}
{\min}\sum_{\layeridex=1}^{\layernum}\|\signalvec\layernot{\layeridex}- 
\adjacencyestmat
\signalvec\layernot{\layeridex}\|_2^2+
\sum_{\layeridex=1}^{\layernum}\frac{\regpar} 
{\samplenum\layernot{\layeridex}}\|\observationvec\layernot{\layeridex}- 
\samplemat\layernot{\layeridex}\signalvec\layernot{\layeridex}\|_2^2
\end{align}
where the regularization terms in~\eqref{eq:topsiginfer} do not appear. 
Clearly,~\eqref{eq:siginfer} conveniently decouples across $\layeridex$ as
\begin{align}
\label{eq:siginferone}
\underset{\signalvec\layernot\layeridex}
{\min}~~\singinfeonefun(\signalvec
\layernot\layeridex):=\tfrac{\samplenum\layernot{\layeridex}} 
{\regpar}\|(\identitymat_\vertexnum-\adjacencyestmat)
\signalvec\layernot{\layeridex}\|_2^2+\|\observationvec\layernot{\layeridex}- 
\samplemat\layernot{\layeridex}\signalvec\layernot{\layeridex}\|_2^2.
\end{align}
The first quadratic in~\eqref{eq:siginferone} can be written as
$
\|\big(\identitymat_\vertexnum-\adjacencyestmat)
\signalvec\layernot{\layeridex}\|_2^2=\sum_{\vertexind=1}^{\vertexnum}(	
\signalfun\vertlayernot{\vertexind}{\layeridex}- 
\sum_{\vertexindp\in\mathcal{N}_\vertexind}\adjacencyestmatentry 
\vertexvertexnot{\vertexind}{\vertexindp} 
\signalfun\vertlayernot{\vertexindp}{\layeridex})^2
$,
and it can be viewed as a regularizer for 
$\signalvec\layernot\layeridex$, promoting graph signals with similar values at 
neighboring 
nodes. %\acom{rewrite, check, justify}.
% \acom{normalized adj then this is 
%laplacian}
Notice that~\eqref{eq:siginferone} may not be strongly convex, since  
$\identitymat_\vertexnum-\adjacencyestmat$ could be rank deficient.
\cmt{gradient}Nonetheless, since 
$\singinfeonefun(\cdot)$ is 
smooth, \eqref{eq:siginferone} can be 
readily solved  via gradient 
descent  (GD) iterations
\begin{align}
\label{eq:gd}
\signalvec\layernot{\layeridex}\gditernot{\gditerind}=	
\signalvec\layernot{\layeridex}\gditernot{\gditerind-1}-\gdstep	
\nabla\singinfeonefun(\signalvec\layernot{\layeridex}\gditernot{\gditerind})
\end{align}
where
$\nabla\singinfeonefun(\signalvec\layernot\layeridex):= 
\tfrac{\samplenum\layernot{\layeridex}}
{\regpar}\big((\identitymat_\vertexnum
-\adjacencyestmat)\transpose\! 
(\identitymat_\vertexnum-\adjacencyestmat)+$
$\!
\samplemat\layernot{\layeridex}\!\transpose\samplemat\layernot{\layeridex}\big) 
\signalvec\layernot\layeridex$ 
$-\samplemat\layernot{\layeridex}\observationvec\layernot{\layeridex}$, and 
$\gdstep>0$ 
is the stepsize chosen e.g. by the Armijo 
rule~\cite{bertsekas1999nonlinear}.  The 
computational cost of
\eqref{eq:gd}  is dominated by the matrix-vector multiplication  of 
$\identitymat_\vertexnum-\adjacencyestmat\iteratenot{\iterateindex}$ with 
$\signalvec\layernot{\layeridex}$, which is proportional to $\mathcal{O}(k_{\rm 
nnz})$, where 
$k_{\rm nnz}$ denotes the 
number of 
non-zero entries of $\adjacencyestmat$. Moreover, the learned  
$\adjacencyestmat$ is expected to be sparse due to the $\ell_1$ 
regularizer in~\eqref{eq:topsiginfer}, which renders first-order 
iterations~\eqref{eq:gd} 
computationally attractive, 
especially when graphs are 
large.
%that the sparsity 
%promoting  Therefore, in the case of sparse 
%networks the gradient descent 
%solver is 
%preferable compared to solving~\eqref{eq:siginferone} in closed form (if 
%possible), which requires 
%the
%inversion of an $\vertexnum\times\vertexnum$ matrix. 
The GD iterations~\eqref{eq:gd} are run in parallel across $\layeridex$ 
until convergence to a minimizer of~\eqref{eq:siginferone}.

\cmt{topology identification algo}
On the other hand, with
$\{\signalestvec\layernot\layeridex\}_{\layeridex=1}^\layernum$  available,  
$\adjacencyestmat$ is found via
\begin{align}
\label{eq:topinfer}
\underset{\adjacencymat\in\adjacencyset} 
{\min}~~
\sum_{\layeridex=1}^{\layernum}\|\signalestvec\layernot{\layeridex}- 
\adjacencymat
\signalestvec\layernot{\layeridex}\|_2^2 
+\regparadjone\|\adjacencymat\|_1+\regparadjtwo 
\|\adjacencymat\|_F^2
%\\
%\st&\adjacencymat\~\adjacencymatentry\vertexvertexnot{\vertexind}{\vertexind}=0,~\vertexind=1,\ldots,\vertexnum
\end{align}
where the LS observation error in~\eqref{eq:topsiginfer} has been omitted. Note 
that~\eqref{eq:topinfer} is strongly convex, and as such it admits a unique 
minimizer. Hence, we adopt the  alternating methods of multipliers (ADMM), 
which guarantees convergence to the global minimum in a finite number of 
iterations; see e.g.~\cite{giannakis2016decentralized}. The derivation of the 
algorithm is omitted due to lack of space; instead  the detailed derivation of 
an ADMM solver for a more general setting will be presented in  Sec. 
\ref{sec:svaralgo}.

The BCD solver for JISG  is summarized as Algorithm~\ref{algo:JISG}. JISG 
converges at least to a stationary point of~\eqref{eq:topsiginfer}, as asserted 
by the ensuing proposition. 
\begin{algorithm}[t]                
\caption{Joint Infer. of Signals and Graphs (JISG)}
\label{algo:JISG}    
\vspace{0.2cm}
\begin{minipage}{40cm}
\indent\textbf{Input:} Observations 
$\{\observationvec\layernot{\layeridex}\}_{\layeridex=1}^\layernum$; sampling 
matrices
$\{\samplemat\layernot{\layeridex}\}_{\layeridex=1}^\layernum$;
\\ \hspace*{1.1cm} and regularization parameters 
$\{\regpar,\regparadjone,\regparadjtwo\}$%; Tolerance $\omega$
\vspace{0.2cm}
\begin{algorithmic}[1]
\STATE\emph{Intialize:} 
$\signalestvec\layernot{\layeridex}\iteratenot{0}= 
\samplemat\layernot{\layeridex}\transpose 
\observationvec\layernot{\layeridex},~\layeridex=1,\ldots,\layernum$
\vspace{0.1cm}
\STATE 	\textbf{while} {iterates not converge}
%				 {\small 
%$\lossfunction\big(\adjacencymat\iteratenot{\iterateindex}, 
%				 
%\{\signalvec\layernot\layeridex\iteratenot{\iterateindex}\}_{\layeridex=1}
%				 	^\layernum 
%				 	
%)$-$\lossfunction\big(\adjacencymat\iteratenot{\iterateindex},\{\signalvec\layernot
% 
%				 	\layeridex\iteratenot{\iterateindex}\}_{\layeridex=1}
%				 	^\layernum \big)\ge\omega$} 
\textbf{do}\hspace{0.1cm}\\		
\STATE \hspace{1cm}Estimate $\adjacencyestmat\iteratenot{\iterateindex}$ 
from~\eqref{eq:topinfer} using ADMM.
\STATE \hspace{1cm}Update 
$\{\signalestvec\layernot\layeridex\iteratenot{\iterateindex}\}_{\layeridex=1}	
^\layernum $ using~\eqref{eq:siginferone} and \eqref{eq:gd}.
\STATE \hspace{1cm}{$\iterateindex=\iterateindex+1$}
\STATE\textbf{end while}
\end{algorithmic}
\vspace{0.2cm}
\indent\textbf{Output:} {$\{\signalestvec\layernot\layeridex\}_ 
{\layeridex=1}^\layernum,\adjacencyestmat$}
.
%\UNTIL{ $\|                          \auxveco\iternot{k+1} 
%-                          
%\trsamplealphavec\iternot{k+1}\| 
%\leq                          \epsilon $ }
%		\end{algorithmic}
\end{minipage}
\end{algorithm}

\begin{mypropositionhere}The sequence  of iterates 
$\big\{\{\signalestvec\layernot\layeridex\iteratenot{\iterateindex}\}_	
{\layeridex=1}^\layernum,\adjacencyestmat\iteratenot{\iterateindex}\big\}_\iterateindex$,
	resulting from obtaining the global minimizers of~\eqref{eq:siginfer}  
and~\eqref{eq:topinfer}, is bounded and converges monotonically to a stationary 
point of~\eqref{eq:topsiginfer}.
\end{mypropositionhere}
\begin{proof}
The basic convergence results of  BCD  have been established 
in~\cite{tseng2001convergence}. First, notice that all the terms 
in~\eqref{eq:topsiginfer} are differentiable over their open domain except the 
non-differentiable $\ell_1$ norm, which is  however separable. 	These 
observations establish, based on~\cite[Lemma 3.1]{tseng2001convergence}, 
that 	
$\lossfunction(\adjacencymat,\{\signalvec\layernot\layeridex\}_{\layeridex=1}
^\layernum )$ is regular at each coordinatewise minimum point 	
$\adjacencyestmat^*,\{\signalestvec\layernot\layeridex^*\}_ 	
{\layeridex=1}^\layernum$, and therefore every such a point is a stationary 
point 
of~\eqref{eq:topsiginfer}. 	Moreover, 
$\lossfunction(\adjacencymat,\{\signalvec\layernot\layeridex\}_{\layeridex=1}
^\layernum )$ is continuous and  convex per variable. 
Hence, by appealing to~\cite[Theorem 5.1]{tseng2001convergence}, the sequence 
of 
iterates 
generated by JISG converges monotonically to  a coordinatewise minimum 
point of $\lossfunction$, and consequently to a stationary point 
of~\eqref{eq:topsiginfer}. 
%\cmt{In mathematics, the Gâteaux differential or Gâteaux derivative is a 
%generalization 
%of the 
%concept 
%of directional derivative in differential calculus.} 
\end{proof} 

\noindent A few remarks are now in order.
\begin{myremarkhere}
A popular alternative to the elastic net regularizer is the nuclear norm 
$\regfun(\adjacencymat)=\|\adjacencymat\|_\ast$ that promotes low rank 
of the learned adjacency matrix -  a well-motivated attribute when the graph is 
expected to  exhibit  clustered structure~\cite{chen2014clustering}.%, i.e., 
%nodes 
%in the graph may belong to different communities.
\end{myremarkhere}
\begin{myremarkhere}
Oftentimes, prior  information  about $\graph$ may be available, e.g. the 
support of 
$\adjacencymat$; nonnegative edge weights 
$\adjacencymatentry\vertexvertexnot{\vertexind}{\vertexindp}\ge0, \forall 
\vertexind,\vertexindp$; or, the  value of
$\adjacencymatentry\vertexvertexnot{\vertexind}{\vertexindp}$ for some 
$\vertexind,\vertexindp$. Such prior information  can be easily incorporated 
in~\eqref{eq:topsiginfer} by adjusting  $\adjacencyset$, and the ADMM solver 
accordingly.
\end{myremarkhere}
\begin{myremarkhere}
The estimator in~\eqref{eq:siginfer} that relies on SEMs 
is capable of estimating  
functions over \emph{directed} graphs as well as undirected ones, while
kernel-based approaches~\cite{smola2003kernels} and estimators that rely 
on the graph Fourier transform~\cite{shuman2013emerging} are usually confined 
to undirected 
graphs.
\end{myremarkhere}

\begin{myremarkhere}
In real-world networks, 	sets of nodes may depend upon each other via 
multiple types of  relationships, which ordinary networks cannot	
capture~\cite{kivela2014multilayer}. Consequently, generalizing the	traditional 
\emph{single-layer} to \emph{multilayer} networks that organize the nodes into 
different groups, called \emph{layers}, is well motivated.  Such layer 
structure can be  incorporated in~\eqref{eq:topsiginfer} via appropriate 
regularization; see e.g. \cite{ioannidis2018multilay}. Thus, the JISG estimator 
can also accommodate multilayer graphs. 
\end{myremarkhere}

\section{Jointly  infer graphs and processes over time}
\label{sec:svar}
Real-world networks often involve processes that vary over time, with 
 dynamics not captured by SEMs. This section considers an 
alternative based on  SVARMs that allows for joint inference of dynamic 
network processes and graphs.

\subsection{Batch Solver for JISG over time}
\label{sec:svaralgo}Given $\{	\observationvec\timenot{\timeind}, 
\samplemat\timenot{\timeind}\}_{t=1}^T$, this section develops  an efficient 
approach to infer $\adjacencyzeromat$, $\adjacencyonemat$, and 
$\{\signalvec\timenot{\timeind} \}_{\timeind=0}^{\timenum}$. Clearly, to cope 
with the undetermined system of equations  \eqref{eq:svarmvec} and 
\eqref{eq:obsvec},  one has to exploit the structure in $\adjacencyzeromat$ and
$\adjacencyonemat$. This prompts the following regularized  
LS  objective
\begin{align}
\label{eq:topsiginfersvarm}
\hspace{-1cm}\underset{\adjacencymat\lagnot{0}\in\adjacencyset,\atop
\adjacencymat\lagnot{1}, \{\signalvec\timenot{\timeind}
\}_{\timeind=0}^{\timenum}
}{\minimize}&
\sum_{\timeind=1}^{\timenum}
\|\signalvec\timenot{\timeind}
-\adjacencymat\lagnot{0}\signalvec\timenot{\timeind}
-\adjacencymat\lagnot{1}\signalvec\timenot{\timeind-1}
\|_2^2\\\hspace{-2cm}+&\|\signalvec\timenot{0}-\zerostatevec\|_2^2+
\sum_{\timeind=1}^{\timenum}
\frac{\regpar} 
{\samplenum\timenot{\timeind}}\|\observationvec 
\timenot{\timeind}- 
\samplemat\timenot{\timeind}\signalvec 
\timenot{\timeind}\|_2^2\nonumber\\
+&\regfunelnet(\adjacencymat\lagnot{0})
+\regfunelnet(\adjacencymat\lagnot{1})\nonumber
\end{align}
where   $\regpar>0$ is a regularization scalar weighting the fit to the 
observations,  and  $\regfunelnet(\adjacencymat
)\define2\regparadjone\|\adjacencymat\|_1+{\regparadjtwo}
\|\adjacencymat\|_F^2\nonumber$ is the elastic net regularizer for the 
connectivity matrices. 
The first sum accounts for the LS fitting error of the SVARM, and the second LS 
cost accounts for the initial conditions. The third term sums the measurement 
error over $\timeind$. Finally,  the elastic net penalty terms 
$\regfunelnet(\adjacencymat\lagnot{0})$, and 
$\regfunelnet(\adjacencymat\lagnot{0})$ favor connections among highly 
correlated nodes; see also discussion after \eqref{eq:topsiginfer}.

The optimization problem in \eqref{eq:topsiginfersvarm} is nonconvex due to the 
bilinear terms $\adjacencymat\lagnot{0}\signalvec\timenot{\timeind}$, and 
$\adjacencymat\lagnot{1}\signalvec\timenot{\timeind-1}$; nevertheless, it is 
convex w.r.t. each of the variables separately. Next, an efficient  algorithm 
based on  BCD is put forth that provably attains a 
stationary point of \eqref{eq:topsiginfersvarm}. With 
$\adjacencyestmat\lagnot{0}$, and $\adjacencyestmat\lagnot{1}$ available, the 
following objective  yields estimates 
\begin{align}
\label{eq:siginfersvarm}
\{\signalestvec\timegiventimenot{\timeind}{\timenum}\}_{\timeind=0}^{\timenum}\define
%&
%\\
&\underset{\{\signalvec\timenot{\timeind}
\}_{\timeind=0}^{\timenum}
}{\arg\min}
\sum_{\timeind=1}^{\timenum}
\|\signalvec\timenot{\timeind}
-\adjacencyestmat\lagnot{0}\signalvec\timenot{\timeind}
-\adjacencyestmat\lagnot{1}\signalvec\timenot{\timeind-1}
\|_2^2\nonumber\\+&\|\signalvec\timenot{0}-\zerostatevec\|_2^2+
\sum_{\timeind=1}^{\timenum}
\frac{\regpar} 
{\samplenum\timenot{\timeind}}\|\observationvec 
\timenot{\timeind}- 
\samplemat\timenot{\timeind}\signalvec 
\timenot{\timeind}\|_2^2%\nonumber
%\\
%\st&\adjacencymat\~\adjacencymatentry\vertexvertexnot{\vertexind}{\vertexind}=0,~\vertexind=1,\ldots,\vertexnum
\end{align}
where $\signalestvec\timegiventimenot{\timeind}{\timenum}$ denotes  the 
estimate of 
$\signalvec\timenot{\timeind}$ given 
$\{\observationvec\timenot{\tau}\}_{\tau=1}^{\timenum}$.   Different from 
\eqref{eq:siginfer},  the time-lagged dependencies $\adjacencymat\lagnot{1} 
\signalvec\timenot{\timeind-1}$ couple the objective in 
\eqref{eq:siginfersvarm} across $\timeind$. Upon defining 
$\auxadj:=\identitymat-\adjacencymat\lagnot{0}$ that is assumed invertible, 
$\transmat:={\auxadj}\inv\adjacencymat\lagnot{1}$, and
$\transcov:=(\auxadj\transpose\auxadj)\inv$, we can express 	
\eqref{eq:siginfersvarm}  equivalently as
\begin{align}
\label{eq:siginfersvarmreform}
\{\signalestvec\timegiventimenot{\timeind}{\timenum}\}_{\timeind=0}^{\timenum} 
\define&%&\\
\underset{\{\signalvec\timenot{\timeind}
\}_{\timeind=0}^{\timenum}
}{\arg\min}
%{\lossfunction(\adjacencymat,\{\signalvec\layernot\layeridex\}_{\layeridex=1}
%	^\layernum ):=} 
\sum_{\timeind=1}^{\timenum}
\|\signalvec\timenot{\timeind} 
-\transmat\signalvec\timenot{\timeind-1}
\|_{\transcov}^2+\nonumber\\+
&\|\signalvec\timenot{0}-\zerostatevec\|_2^2+\sum_{\timeind=1}^{\timenum}
\frac{\regpar} 
{\samplenum\timenot{\timeind}}\|\observationvec 
\timenot{\timeind}- 
\samplemat\timenot{\timeind}\signalvec 
\timenot{\timeind}\|_2^2.%\nonumber
%\\
%\st&\adjacencymat\~\adjacencymatentry\vertexvertexnot{\vertexind}{\vertexind}=0,~\vertexind=1,\ldots,\vertexnum
\end{align}
The minimizer of \eqref{eq:siginfersvarmreform} can be attained by
\begin{align}
\hat{\signalfulvec}:=\underset
{\signalfulvec} {\text{argmin}} %\left
\Bigg\| 
\overset
{\observationfulvec
}{\overbrace{\begin{bmatrix}
\zerostatevec\\\mathbf{0}_\vertexnum\\\vdots\\
\mathbf{0}_\vertexnum\\
\hline
\observationvec\timenot{1}\\
\vdots\\
\observationvec\timenot{\timenum}
\end{bmatrix}}}-
\overset
{\transfulmat
}{\overbrace{
\begin{bmatrix}
\identitymat_N\\
-\transmat&\identitymat_N\\
&\ddots&\ddots\\
&&-\transmat&\identitymat_N\\
%\noindent\rule{5cm}{1pt}
\hline
\mathbf{0}&\samplemat\timenot{1}\\
\vdots&&\ddots\\
\mathbf{0}&&&\samplemat\timenot{\timenum}
\end{bmatrix}}}
\overset
{\signalfulvec
}{\overbrace{\begin{bmatrix}
\signalvec\timenot{0}\\
\signalvec\timenot{1}\\
\vdots\\
\signalvec\timenot{\timenum}
\end{bmatrix}}}%\right
\Bigg\|_{\corfulmat}^2
\label{eq:extls}
\end{align}
where the square matrix $\corfulmat$ weighting the norm in \eqref{eq:extls}
has block entries 
$[\corfulmat]_{1,1}=\identitymat_\vertexnum$,
$\{[\corfulmat]_{i,i}=\transcov\}_{i=2}^{T+1}$, 
$\{[\corfulmat]_{i,i}=(\samplenum\timenot{i-\timenum-1}/\regpar)
\identitymat_{\samplenum\timenot{i-T-1}}
\}_{i=T+2}^{2T+1}$. The minimizer of \eqref{eq:extls} admits a closed-form 
solution as 
$\hat{\signalfulvec}=(\transfulmat\transpose\corfulmat\inv\transfulmat)\inv
\transfulmat\transpose\corfulmat\inv\observationfulvec$.
Unfortunately, %such and approach does not exploit the special structure of  
%\acom{arg1} 
this direct 
approach incurs computational complexity $\mathcal{O}(\vertexnum^3\timenum^3)$, 
which  does not scale favorably. Scalability in $\vertexnum$ can be aided by distributed solvers, which are possible using consensus-based ADMM iterations; see e.g.,~\cite{giannakis2016decentralized}. With regards to scalability across time,  the following proposition establishes that an 
iterative algorithm attains 
%the minimizers
$	\{\signalestvec\timegiventimenot{\timeind}{\timenum}\}_{\timeind=0}^{\timenum}$ with 
complexity that is linear in $\timenum$. 

Since \eqref{eq:siginfersvarmreform} is identical to the deterministic 
formulation of the Rauch-Tung-Striebel (RTS) smoother for a state-space model 
with state noise covariance $\transcov$ and measurement noise covariance 
$\{\samplenum\timenot{\timeind}/\regpar
\identitymat_{\samplenum\timenot{\timeind}}\}_{\timeind}$, we deduce that the 
RTS algorithm,  see e.g. \cite{anderson1979optimal},\cite{rauch1965dynamic}, 
applies readily to obtain sequentially the structured per slot $\timeind$ 
component $	\{\signalestvec\timegiventimenot 
{\timeind}{\timenum}\}_{\timeind=0}^{\timenum}$. Summing up, we have 
established the following result. 

\begin{mypropositionhere}\thlabel{thm:rts}
%	With $\auxadj:=\identitymat-\adjacencymat\lagnot{0}$ invertible, % and  with 
%	$\transmat:={\auxadj}\inv\adjacencymat\lagnot{1}$, and
%	$\transcov:=(\auxadj\transpose\auxadj)\inv$,  the estimates
The minimizers of  \eqref{eq:siginfersvarm}  $	
\{\signalestvec\timegiventimenot{\timeind}{\timenum}\}_{\timeind=0}^{\timenum}$ 
are given iteratively by the RTS
 smoother summarized as Algorithm~\ref{algo:GRTS}.
\end{mypropositionhere}

\begin{algorithm}[t]                
\caption{{RTS smoother}}
\label{algo:GRTS}    
\begin{minipage}{40cm}
\indent\textbf{Input.} $\transmat$; $\transcov$; 
$\signalestvec\timegiventimenot{0}{0}$; $
\coverrror\timegiventimenot{0}{0}$;	
$\{\observationvec\timenot{\timeind}\}_{\timeind=1}^{\timenum}$; 
$\{\samplemat\timenot{\timeind}\}_{\timeind=1}^{\timenum}$\vspace{0.1cm}
\\
%\indent \hspace{0.5cm} \textbf{Initialization.}  
%$\signalestvec\timegiventimenot{0}{0}=\zerostatevec, 
%\coverrror\timegiventimenot{0}{0}=\identitymat_\vertexnum$
%\\
\indent	\hspace{0.5cm} \textbf{for}~{$\timeind=1,2,\ldots,T$}{  }\textbf{do}
\\
KF1. \hspace{0.5cm} 
$\signalestvec\timegiventimenot{\timeind}{\timeind-1}=\transmat\signalestvec 
\timegiventimenot{\timeind-1}{\timeind-1}$\\
KF2. \hspace{0.5cm} $\coverrror\timegiventimenot{\timeind}{\timeind-1}= 
\transmat\coverrror\timegiventimenot{\timeind-1}{\timeind-1} 
\transmat\transpose + 
\transcov$\\
KF3. 	\hspace{0.5cm} 
$\kalmanfiltgain\timenot{\timeind}= 
\coverrror\timegiventimenot{\timeind}{\timeind-1} 
\samplemat\timenot{\timeind}\transpose\big((\samplenum\timenot{\timeind}/\regpar)
\identitymat_{\samplenum\timenot{\timeind}} + 
\samplemat\timenot{\timeind}\coverrror\timegiventimenot{\timeind}{\timeind-1} 
\samplemat\timenot{\timeind}\transpose\big)\inv$
\\
KF4. \hspace{0.5cm} 
$\coverrror\timegiventimenot{\timeind}{\timeind}=(\identitymat_\vertexnum 
-\kalmanfiltgain\timenot{\timeind}\samplemat\timenot{\timeind})
\coverrror\timegiventimenot{\timeind}{\timeind-1}$
\\
KF5. \hspace{0.5cm}  $\signalestvec\timegiventimenot{\timeind}{\timeind}= 
\signalestvec\timegiventimenot{\timeind}{\timeind-1} 
+\kalmanfiltgain\timenot{\timeind}
(\observationvec\timenot{\timeind}-\samplemat\timenot{\timeind}
\signalestvec\timegiventimenot{\timeind}{\timeind-1})$
\\
\indent	\hspace{0.5cm} \textbf{end for}
\\
\indent	\hspace{0.5cm} \textbf{for}~{$\timeind=T-1,T-2,\ldots,1$}{  }\textbf{do}
\\
KS1. \hspace{0.5cm}
$\kalmansmoothgain\timenot{\timeind}=\coverrror
\timegiventimenot{\timeind}{\timeind}\transmat\transpose\coverrror
\timegiventimenot{\timeind+1}{\timeind}\inv$ 
\\
KS2. \hspace{0.5cm}  $\signalestvec\timegiventimenot{\timeind}{\timenum}= 
\signalestvec\timegiventimenot{\timeind}{\timeind} 
+\kalmansmoothgain\timenot{\timeind}
(\signalestvec\timegiventimenot{\timeind+1}{\timenum}-
\signalestvec\timegiventimenot{\timeind+1}{\timeind})$
\\
\indent	\hspace{0.5cm} \textbf{end for}\\
\indent\textbf{Output.} 
$\{\signalestvec\timegiventimenot{\timeind}{\timenum}\}_{\timeind=0}^{\timenum} 
$
\end{minipage}
\end{algorithm}

Algorithm \ref{algo:GRTS} is a forward-backward algorithm: The forward 
direction is executed by  Kalman filtering (steps KF1-KF5); and the backward 
direction is performed by  Kalman smoothing (steps KS1-KS2) \cite{rauch1965dynamic}. 
%KF2-KF4 specify $\coverrror\timegiventimenot{\timeind}{\timeind-1}$, 
%$\coverrror\timegiventimenot{\timeind}{\timeind}$, $\kalmanfiltgain\timenot{\timeind}$ that are known in the KF literature as the mean square-error matrices for prediction, correction, and Kalman gain matrix. 
The algorithm smooths the state estimates over the interval $[1,\timenum]$.  
Each step incurs  complexity at most 
$\mathcal{O}(\vertexnum^3)$, and hence the overall complexity for estimating 
$\{\signalestvec 
\timegiventimenot{\timeind}{\timenum}\}_{\timeind=0}^{\timenum} 
$ is $\mathcal{O}(\vertexnum^3\timenum)$, which scales 
favorably for large $\timenum$.
For solving \eqref{eq:siginfersvarm}, one should initialize the RTS by $\signalestvec\timegiventimenot{0}{0}=\zerostatevec$, and $ 
\coverrror\timegiventimenot{0}{0}=\identitymat_\vertexnum$.

To estimate the adjacency matrices given  $\{\signalestvec
\timenot{\timeind}\}_{\timeind=1}^{\timenum}$, consider the following problem
\begin{align}
\label{eq:topinfersvarm}
\hspace{-0.5cm}\underset{\adjacencymat\lagnot{0}\in\adjacencyset, 
\adjacencymat\lagnot{1}
}{\minimize}&
\sum_{\timeind=1}^{\timenum}
\|\signalestvec\timenot{\timeind}
-\adjacencymat\lagnot{0} 
\signalestvec\timenot{\timeind}
-\adjacencymat\lagnot{1} 
\signalestvec\timenot{\timeind-1}
\|_2^2\nonumber\\
%+\regparadj\regfun(\adjacencymat)
+&
\regfunelnet(\adjacencymat\lagnot{0})
+\regfunelnet(\adjacencymat\lagnot{1}).
\end{align}
The objective in \eqref{eq:topinfersvarm} is convex albeit nonsmooth, and hence 
ADMM can be adopted to obtain
$\adjacencyestmat\lagnot{0}$ and $ \adjacencyestmat\lagnot{1}$. The ADMM solver is summarized as Algorithm~\ref{algo:svarmadmm}, and its derivation is deferred to the Appendix. 

The overall procedure for joint inference of  signals and graphs over time (JISGoT) is tabulated as Algorithm \ref{algo:DJISG}. Convergence of JISGoT is asserted in the following proposition, the proof of which is similar to Proposition 1, and hence is omitted.

\begin{mypropositionhere} The sequence  of iterates 
$\big\{\{\signalestvec\timenot{\timeind}\iteratenot{\iterateindex}\}_ 
{\timeind=0}^{\timenum},$ $\adjacencyestmat\lagnot{0}\iteratenot{\iterateindex},
\adjacencyestmat\lagnot{1}\iteratenot{\iterateindex}\big\}_\iterateindex$,
resulting from obtaining the global minimizers of~\eqref{eq:siginfersvarm}  
and~\eqref{eq:topinfersvarm}, is bounded and converges 
monotonically 
to a stationary point of~\eqref{eq:topsiginfersvarm}.
\end{mypropositionhere}

\begin{algorithm}[t]                
\caption{ADMM for topology inference}
\label{algo:svarmadmm}    
\vspace{0.2cm}
\begin{minipage}{40cm}
\indent\textbf{Input:} 
$\{\signalestvec\timenot{\timeind}\}_{\timeind=0}^{\timenum}$, 
$\regparadjone,\regparadjtwo$
\vspace{0.2cm}
\begin{algorithmic}[1]
\STATE
\textbf{Initialization.}  Initialize all variables to 
zero. 
%			$\adjacencymat\lagnot{1}$
\vspace{0.1cm}
\STATE 	\textbf{while} {iterates not converge}
\textbf{do}\hspace{0.1cm}\\		
\STATE \hspace{1cm}Update  
$\adjacencymat\lagnot{0}$ using \eqref{eq:ad0update}.
\STATE \hspace{1cm}Update  
$\adjacencymat\lagnot{1}$ using \eqref{eq:ad1update}.
\STATE \hspace{1cm}Update auxiliary variable 
$\adjacencymathelp\lagnot{0}$ using \eqref{eq:adhelp0update}.
\STATE \hspace{1cm}Update auxiliary variable 
$\adjacencymathelp\lagnot{1}$ using \eqref{eq:adhelp1update}. 
\STATE \hspace{1cm}Update Lagrange multipliers using 
\eqref{eq:helpupdate}. 
\STATE\textbf{end while}
\end{algorithmic}
\vspace{0.2cm}
\indent\textbf{Output:} {$\adjacencymat\lagnot{0}, 
\adjacencymat\lagnot{1}$}
.
%\UNTIL{ $\|                          \auxveco\iternot{k+1} 
%-                          
%\trsamplealphavec\iternot{k+1}\| 
%\leq                          \epsilon $ }
%		\end{algorithmic}
\end{minipage}
\end{algorithm}

\begin{algorithm}[t]                
\caption{JISG over time (JISGoT)}
\label{algo:DJISG}    
\vspace{0.2cm}
\begin{minipage}{40cm}
\indent\textbf{Input:} Observations 
$\{\observationvec\layernot{\timeind}\}_{\timeind=1}^{\timenum}$; sampling 
matrices
$\{\samplemat\layernot{\timeind}\}_{\timeind=1}^{\timenum}$;
\\ \hspace*{1.1cm} and regularization parameters 
$\{\regpar,\regparadjone,\regparadjtwo\}$%; Tolerance $\omega$
\vspace{0.2cm}
\begin{algorithmic}[1]
\STATE\emph{Intialize:} 
$\signalestvec\layernot{\timeind}\iteratenot{0}= 
\samplemat\layernot{\timeind}\transpose 
\observationvec\layernot{\timeind},~\timeind=1,\ldots,\timenum$
\vspace{0.1cm}
\STATE 	\textbf{while} {iterates not converge}
\textbf{do}\hspace{0.1cm}\\		
\STATE \hspace{1cm}Estimate 
$\adjacencyestmat\lagnot{0}\iteratenot{\iterateindex}$, and 
$\adjacencyestmat\lagnot{1}\iteratenot{\iterateindex}$ from Algorithm 
\ref{algo:svarmadmm}.
\STATE \hspace{1cm}Update 
$\{\signalestvec\layernot\timeind \iteratenot{\iterateindex}\}_{\timeind=0}	
^{\timenum} $ from Algorithm \ref{algo:GRTS}.
\STATE \hspace{1cm}{$\iterateindex=\iterateindex+1$}
\STATE\textbf{end while}
\end{algorithmic}
\vspace{0.2cm}
\indent\textbf{Output:} {$\{\signalestvec\layernot\timeind\}_ 
{\timeind=0}^{\timenum},\adjacencyestmat\lagnot{0}, \adjacencyestmat\lagnot{1}$}
.
%\UNTIL{ $\|                          \auxveco\iternot{k+1} 
%-                          
%\trsamplealphavec\iternot{k+1}\| 
%\leq                          \epsilon $ }
%		\end{algorithmic}
\end{minipage}
\end{algorithm}

%RTS performs \emph{fixed-interval} smoothing, since the whole batch 
%$\{\observationvec\timenot{\timeind}\}_{\timeind=1}^{\timenum}$ has to be 
%available to learn $\{\signalestvec 
%\timegiventimenot{\timeind}{\timenum}\}_{\timeind=1}^{\timenum}$.

\subsection{Fixed-lag solver for online JISGoT}
JISGoT performs \emph{fixed-interval} smoothing, since the whole batch 
$\{\observationvec\timenot{\timeind}\}_{\timeind=1}^{\timenum}$ has to be 
available to learn $\{\signalestvec 
\timegiventimenot{\timeind}{\timenum}\}_{\timeind=0}^{\timenum}$.
%\cmt{motivation paragraph}A major limitation of the  
%algorithm~\eqref{algo} is that all the data have to be available before 
%estimating $\{\signalvec\timenot{\timeind}\}_{\timeind=1}^{\timenum}$. 
Albeit 
useful for applications such as processing electroencephalograms, and analysis 
of historical trade transaction data, this batch solver is not suitable 
for online applications, such as stock market prediction, analysis of online social networks, and propagation of cascades over interdependent power networks. Such applications  enforce strict delay constraints and require estimates within a window or fixed lag \cite{anderson1979optimal}.
%$\lagnum$.

Driven by the aforementioned delay constraints, the goal here is to 
estimate $\signalvec\timenot{\timeind}, \adjacencymat\lagnot{0}, \adjacencymat\lagnot{1}$, 
relying upon observations  up to time $\timepluslagnum:=\timeind+\lagnum$,  with $\lagnum$ denoting the  affordable 
delay window length. 
%	Notice that the resulting graph estimates
%	$\adjacencyestmat\lagnot{1}\timenot{\timepluslagnum}$ and $
%		\adjacencyestmat\lagnot{0}\timenot{\timepluslagnum}$ vary over time, which is 
%	 well-motivated when the underlying processes are non-stationary.
%The online algorithm will generate   These time-varying estimates \redb{ 
%Moreover, the topology might change over time. this is}  
Supposing that a KF (cf. Algorithm \ref{algo:GRTS}) has been run up to time 
$\timeind$ to yield estimates $\signalestvec\timegiventimenot 
{\timeind}{\timeind}$ and $\errorcovkf\timegiventimenot{\timeind}{\timeind}$, 
the desired estimates $\adjacencyestmat\lagnot{1}\timenot{\timepluslagnum},
\adjacencyestmat\lagnot{0}\timenot{\timepluslagnum},
\{\signalestvec\timegiventimenot{\tau}{\timepluslagnum}\}_{\tau=\timeind}^
{\timepluslagnum}$ can be obtained at $\timepluslagnum$ by solving the 
following problem
\begin{align}
\label{eq:topsiginfersvarmfixedlag}
&\underset{\adjacencymat\lagnot{0},
\adjacencymat\lagnot{1}\in\adjacencyset,\atop
\{\signalvec\timenot{\tau}
\}_{\tau=\timeind}^{\timepluslagnum}
}{\arg\min}
%{\lossfunction(\adjacencymat,\{\signalvec\layernot\layeridex\}_{\layeridex=1}
%	^\layernum ):=} 
\!\sum_{\tau=\timeind+1}^{\timepluslagnum}
\|\signalvec\timenot{\tau}
-\adjacencymat\lagnot{0}\signalvec\timenot{\tau}
-\adjacencymat\lagnot{1}\signalvec\timenot{\tau-1}
\|_2^2\\
&\hspace{0cm}+\|\signalvec\timenot{\timeind}- 
\signalestvec\timegiventimenot 
{\timeind}{\timeind}\|^2_{\errorcovkf 
	\timegiventimenot{\timeind}{\timeind}}+
\sum_{\tau=\timeind+1}^{\timepluslagnum}
\frac{\regpar} 
{\samplenum\timenot{\tau}}\|\observationvec 
\timenot{\tau}- 
\samplemat\timenot{\tau}\signalvec 
\timenot{\tau}\|_2^2+\regfunelnet(\adjacencymat\lagnot{0})\nonumber\\
%+\regparadj\regfun(\adjacencymat)
&\hspace{0cm}
+\regfunelnet(\adjacencymat\lagnot{1})
+\regsmoothadj\|\adjacencymat\lagnot{0}-\adjacencyestmat\lagnot{0}\timenot{\timepluslagnum -1}\|_F^2
+\regsmoothadj\|\adjacencymat\lagnot{1}-\adjacencyestmat\lagnot{1}\timenot{\timepluslagnum -1}\|_F^2
\nonumber
%\\
%\st&\adjacencymat\~\adjacencymatentry\vertexvertexnot{\vertexind}{\vertexind}=0,~\vertexind=1,\ldots,\vertexnum
\end{align}
where $\adjacencyestmat\lagnot{0}\timenot{\timepluslagnum-1}$ and $\adjacencyestmat\lagnot{1}\timenot{\timepluslagnum-1}$ are the solutions to \eqref{eq:topsiginfersvarmfixedlag} at $\timepluslagnum-1$, and $\regsmoothadj$ controls the effect of the LS terms that promote slow-varying $\{\adjacencyestmat\lagnot{l}\timenot{\timepluslagnum}\}_{l=0,1}$. Similar to \eqref{eq:topsiginfersvarm},
the fixed-lag objective \eqref{eq:topsiginfersvarmfixedlag}   will be solved 
via  a BCD algorithm to a stationary point. Observe that 
solving~\eqref{eq:topsiginfersvarmfixedlag} for 
$\{\signalestvec\timegiventimenot{\tau}{\timepluslagnum}\}_{\tau=\timeind}^
{\timepluslagnum}$ is a special case of the fixed-interval objective in~\eqref{eq:siginfersvarm}, when the initial condition on the state, namely $\signalestvec\timegiventimenot 
{\timeind}{\timeind}$, and ${\errorcovkf\timegiventimenot{\timeind}{\timeind}}$,
are given by the RTS algorithm, and the state is smoothed over the interval $[\timeind,\timepluslagnum]$; see \cite{anderson1979optimal} for details on the fixed-lag smoother. Solving for $\{\adjacencyestmat\lagnot{l}\timenot{\timepluslagnum}\}_{l=0,1}$ entails the additional quadratic terms $\{
\|\adjacencymat\lagnot{l}-\adjacencyestmat\lagnot{l}\timenot{\timepluslagnum-1}\|_F^2
\}_{l=0,1}$ relative to \eqref{eq:topinfersvarm} that can be easily incorporated in Algorithm~\ref{algo:svarmadmm}.

% and 
%$\|\adjacencymat\lagnot{0}-\adjacencyestmat\lagnot{0}\timenot{\timeind-1}\|_F^2$,
%$\|\adjacencymat\lagnot{1}-\adjacencyestmat\lagnot{1}\timenot{\timeind-1}\|_F^2$
%the
% quadratic  terms 

Thus,  one can employ Algorithm \ref{algo:DJISG} with minor modifications to solve the fixed-lag objective \eqref{eq:topsiginfersvarmfixedlag} and estimate
$\{\adjacencyestmat\lagnot{1}\timenot{\timepluslagnum},
\adjacencyestmat\lagnot{0}\timenot{\timepluslagnum},
\{\signalestvec\timegiventimenot{\tau}{\timepluslagnum}\}_{\tau=\timeind}^
{\timepluslagnum}\}$. % at complexity $\mathcal{O}(\lagnum\vertexnum^3)$.
As a convenient byproduct, the novel online 
estimator~\eqref{eq:topsiginfersvarmfixedlag} tracks dynamic topologies from 
the time-varying estimates $\adjacencyestmat\lagnot{1}\timenot{\timepluslagnum}$ and $
\adjacencyestmat\lagnot{0}\timenot{\timepluslagnum}$. 
This is well-motivated when the process is non-stationary, and the underlying 
topologies change over time.

%\begin{myremarkhere}
%	The novel estimator in~\eqref{eq:topinfersvarmfixedlag} %for $\lagnum=1$ 
%	performs online 
%	topology identification based on the observations $\{\observationvec\timenot{\tau} 
%	\}_{\tau=1}^{\tau+\lagnum}$ \acom{clarify}.
%\end{myremarkhere}
\begin{myremarkhere}
Although this section builds upon the SVARM in
\eqref{eq:svarmvec} that accounts only for a single time-lag, the 
proposed algorithms can be readily extended to accommodate SVARMs with multiple time-lags, 
i.e.,  
$\signalvec\timenot{\timeind}
=%\adjacencymat\lagnot{0}\signalvec\timenot{\timeind}
\sum_{\xi=0}^{\Xi}\adjacencymat\lagnot{\xi}\signalvec\timenot{\timeind- \xi}
%+\exogenousweightmat\exogenousvec\layernot{\layeridex} 
+\semnoisevec\timenot{\timeind}
$.
By defining 
the $\Xi\vertexnum\times 1$ extended vector process 
$\extsignalvec\timenot{\timeind}=[\signalvec\timenot{\timeind}\transpose, 
\signalvec\timenot{\timeind-1}\transpose, 
\ldots,\signalvec\timenot{\timeind-\Xi+1}\transpose]\transpose$, 
the $\Xi\vertexnum\times \Xi\vertexnum$ block matrix 
$\extadjacencymat\lagnot{1}$, with 
entries 
$\extadjacencymat\vertexvertexnot{1}{\xi}\lagnot{1}=\adjacencymat 
\lagnot{\xi},~\xi=1,\ldots,\Xi$,
and zero otherwise, the $\Xi\vertexnum\times \Xi\vertexnum$ 
block matrix $\extadjacencymat\lagnot{0}$, with 
entries 
$\extadjacencymat\vertexvertexnot{1}{1} 
\lagnot{0}=\adjacencymat\lagnot{0}$,
$\extadjacencymat\vertexvertexnot{\vertexind} 
{\vertexind}=\identitymat_\vertexnum$ 
for $
\vertexind=2,\ldots,\Xi$ 
and zero otherwise, and the $\Xi\vertexnum\times 1$ error 
vector $\bar{\semnoisevec}\timenot{\timeind}= 
[\semnoisevec\timenot{\timeind}\transpose,\bm 0\transpose,\ldots, \bm 0\transpose]\transpose$, the general SVARM can be 
written as 	$\extsignalvec\timenot{\timeind}=\extadjacencymat\lagnot{0} 
\extsignalvec\timenot{\timeind}+ \extadjacencymat\lagnot{1} 
\extsignalvec\timenot{\timeind-1} 
+\bar{\semnoisevec}\timenot{\timeind}$, which resembles a single time-lag 
SVARM in~\eqref{eq:svarmvec}. 
%Therefore,  the 
%algorithms derived in this 
%section can readily 
%accommodate the more general higher-order SVARM.
\end{myremarkhere}

\section{Identifiability analysis}
\label{sec:identanal}
This section provides results on the identifiability of the network topology given observations  at subsets of nodes under the noise-free SEM. Specifically, in the absence of noise \eqref{eq:semvec} and \eqref{eq:obsvec}  with  $\semnoisevec\layernot{\layeridex}=\observationnoisevec \layernot{\layeridex}=\boldsymbol{0}$,  can be  written as 
\begin{subequations}
\begin{align}
\label{eq:semvecext}
\signalvec\layernot{\layeridex} &=\adjacencymat\signalvec\layernot{\layeridex}\\
%+\exogenousweightmat\exogenousvec\layernot{\layeridex} 
\label{eq:obsvecext}
\observationvecwithmis\layernot
{\layeridex}&=\samplematwithmis\layernot{\layeridex} 
\signalvec\layernot{\layeridex},~~~\layeridex=1,\ldots,\layernum
\end{align}
\end{subequations}
where $\samplematwithmis\layernot{\layeridex}\in\{0,1\}^{N\times N}$ with 
$[\samplematwithmis\layernot{\layeridex}]\vertexvertexnot{\vertexind}{\vertexind}:= \samplematwithmisentry{\vertexind}{\vertexind}{\layeridex}=1$
if 
$\vertexind\in\sampleset\layernot{\layeridex}$, and zero otherwise, 
and $\observationvecwithmis\layernot{\layeridex}\in\rfield^{N}$  with 
$\observationfunwithmis
\vertexnot{\vertexind}\layernot{\layeridex}= \signalfun\vertexnot{\vertexind}\layernot{\layeridex}$ if 
$\vertexind\in\sampleset\layernot{\layeridex}$, and zero if $\vertexind$ is not sampled at $\layeridex$. The $\vertexnum\times\layernum$ matrix
$\observationmatwithmis:=[\observationvecwithmis\layernot{1},
\ldots, \observationvecwithmis\layernot{\layernum}]\transpose$ collects all the observations.

\begin{mydefinition}
\label{def:krus}
The Kruskal rank of an $N\times M$ matrix $\boldsymbol{Y}$ (denoted
hereafter as ${\rm Kruskal}(\boldsymbol{Y})$) is defined as the maximum number $k$ such
that any combination of $k$ columns of $\boldsymbol{Y}$
constitutes a full rank
matrix.
\end{mydefinition}

\begin{figure}[t]
\centering
\includegraphics[width=7.cm]{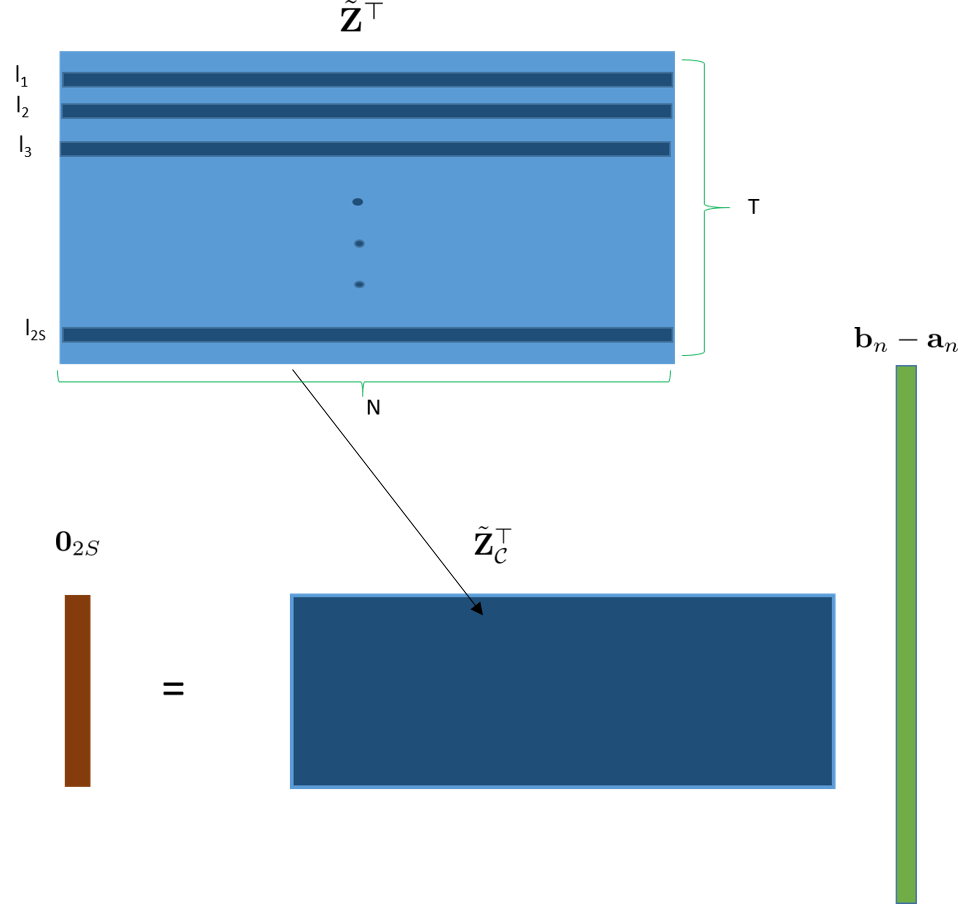}
\caption{Illustration of the matrices in  Theorem \ref{th:identsem} 
}\label{fig:thm1}
\end{figure}
To establish identifiability results, we rely on a couple of assumptions.

\noindent\textbf{as1.}  Matrix 
$\adjacencymat$
has at most $\sparsitynum$ non-zero entries per row.\\
%meaning each node is at most connected to $\sparsitynum$ \red{in-neighbors}.\redb{double check}\\
\noindent\textbf{as2.} 
There exists a subset of columns indexed by $\obscolumset\define{\{\layeridex_1,\layeridex_2,\ldots,
\layeridex_{2\sparsitynum}\}}, 1\le\layeridex_i\le\layernum$, such that the 
coresponding
$\vertexnum\times2\sparsitynum$ sub-matrix 
$\observationmatwithmis_{\obscolumset}:=[\observationvecwithmis\layernot{\layeridex_1},
\ldots, 
\observationvecwithmis\layernot{\layeridex_{2\sparsitynum}}]\transpose$  is fully observable, i.e.,
$\samplematwithmis\layernot{\layeridex}=\identitymat_N,$  and $\observationvecwithmis\layernot{\layeridex}=\signalvec\layernot{\layeridex},$
$\forall\layeridex \in\obscolumset$,   and  ${\rm Kruskal}(\observationmatwithmis_{\obscolumset}\transpose)\geq 2\sparsitynum$.

\begin{mytheoremhere}
Under as1 and as2, the adjacency matrix $\adjacencymat$ can be uniquely 
identified from  
\eqref{eq:semvecext} and \eqref{eq:obsvecext}.
\label{th:identsem}
\end{mytheoremhere}
\begin{proof}
Combining  \eqref{eq:semvecext} and \eqref{eq:obsvecext} leads to
\begin{align}
\label{eq:semobsvecext}
\observationvecwithmis\layernot{\layeridex}&=\samplematwithmis
\layernot{\layeridex} 
\adjacencymat\signalvec\layernot{\layeridex},~~~\layeridex=1,\ldots,\layernum.
\end{align}
Considering the equations indexed by $\layeridex\in \obscolumset$ along with as2, 
we have %of \eqref{eq:semobsvecext} it follows
%\begin{align}
%\label{eq:semonlyobsvecext}
$
\observationvecwithmis\layernot{\layeridex}=
\adjacencymat\observationvecwithmis\layernot{\layeridex},~\layeridex\in\obscolumset
$,
%\end{align}
which leads to the matrix form 
\begin{align}
\label{eq:combin}
\observationmatwithmis_{\obscolumset}=
\adjacencymat\observationmatwithmis_{\obscolumset}.
%\signalvec\layernot{1}, \ldots, 
%\samplematwithmis\layernot{\layernum}\adjacencymat\signalvec\layernot{\layernum}].
\end{align}
Letting
$\observationmatwithmisrow\vertexnot{\vertexind}\transpose$ and 
$\adjacencymatrow\vertexnot{\vertexind}\transpose$ denote the $n$-th rows of 
$\observationmatwithmis_{\obscolumset}$ and $\adjacencymat$ respectively, the 
row-wise version of \eqref{eq:combin} can be written as
\begin{align}
\label{eq:combinexp}
\observationmatwithmisrow\vertexnot{\vertexind}
=\observationmatwithmis_{\obscolumset}\transpose\adjacencymatrow\vertexnot{\vertexind}.
\end{align}
%where the third equation in \eqref{eq:combinexp} follows 
%by  definition,
% of $\signalmat$ and 
%$\samplematwithmis'\vertexnot{\vertexind}$
%and 
Suppose there exists 
a vector 
$\adjacencymatrowother\vertexnot{\vertexind}\in\rfield^{\vertexnum\times1}$ 
with $\sparsitynum$ 
nonzero entries, and 
$\adjacencymatrowother\vertexnot{\vertexind}\ne\adjacencymatrow\vertexnot{\vertexind}$,
 such 
that the 
following holds
\begin{align}
\label{eq:combinexprowdi}
\observationmatwithmisrow\vertexnot{\vertexind}
=\observationmatwithmis_{\obscolumset}\transpose\adjacencymatrowother\vertexnot{\vertexind}.
\end{align}
Combining \eqref{eq:combinexprowdi} and \eqref{eq:combinexp}, we arrive at
\begin{align}
\label{eq:conditionident}
\mathbf{0}_{2\sparsitynum}
=\observationmatwithmis_{\obscolumset}\transpose
(\adjacencymatrowother\vertexnot{\vertexind}-\adjacencymatrow 
\vertexnot{\vertexind}).
\end{align}
Since 
$\adjacencymatrow\vertexnot{\vertexind}$ and $
\adjacencymatrowother\vertexnot{\vertexind}$ both have $\sparsitynum$ nonzero 
entries, vector
$(\adjacencymatrowother\vertexnot{\vertexind}-\adjacencymatrow 
\vertexnot{\vertexind})$ has at most $2\sparsitynum$ nonzero 
entries (see Fig. \ref{fig:thm1}). According to as2, it holds that
${\rm Kruskal}(\observationmatwithmis_{\obscolumset}\transpose)=2\sparsitynum$, and thus any 
$2\sparsitynum$  columns of $\observationmatwithmis_{\obscolumset}\transpose$ 
are	linearly independent; see Definition \ref{def:krus}, which implies $
\adjacencymatrowother\vertexnot{\vertexind}=\adjacencymatrow 
\vertexnot{\vertexind}$ and henceforth leads to a contradiction. The argument 
applies for all $\vertexind=1,\ldots,\vertexnum$.
\end{proof}

Theorem \ref{th:identsem} establishes the sufficient condition for identifying  the network structure given full observations for some $\layeridex$. As1 effectively reduces the number of unknowns to $\sparsitynum\vertexnum$ at most. As2 asserts that  the observation matrix $\observationmatwithmis$ is expressive enough to identify $\adjacencymatrow\vertexnot{\vertexind}$ uniquely.
However, if one does not have control over the sampling process, $\observationvecwithmis \layernot{\layeridex}$ may not be observable at all nodes per slot $\layeridex$. Therefore, the present paper further investigates the identifiability conditions when only partial observations of the network process are available per slot $\layeridex$. The results built upon the following assumption.

\noindent\textbf{as3.} For any subset of 
row indices   $\obsrowset\define{\{\vertexind_1,\vertexind_2,\ldots,
\vertexind_{2\sparsitynum}\}}$ with $ 1\le\vertexind_i\le\vertexnum$, there 
exists a subset of  column indices   
$\obscolumset\define{\{\layeridex_1,\layeridex_2,\ldots,
\layeridex_{2\sparsitynum}\}}$ with $1\le\layeridex_i\le\layernum$  that forms 
the $2\sparsitynum\times2\sparsitynum$ matrix 
$\observationmatwithmis_{\obscolumset\obsrowset}$ 
with entries $[\observationmatwithmis_{\obscolumset\obsrowset}]_{i,j}=
\observationfunwithmis\vertexnot{\vertexind_{i}}\layernot{\layeridex_{j}}$ that is fully observable; meaning 
$\observationfunwithmis\vertexnot{\vertexind_{i}}\layernot{\layeridex_{j}} =
\signalfun\vertexnot{\vertexind_{i}}\layernot{\layeridex_{j}}$, and
$[\samplematwithmis\layernot{\layeridex_{j}}]\vertexvertexnot{\vertexind_{i}}{\vertexind_{i}} = \samplematwithmisentry{\vertexind_{i}}{\vertexind_{i}}{\layeridex_{j}}=1,  
\forall \vertexind_{i}\in \obsrowset,\layeridex_{j} \in \obscolumset$, and satisfies 
${\rm Kruskal}(\observationmatwithmis_{\obscolumset\obsrowset}\transpose)=2\sparsitynum$.
\begin{figure}[t]
\centering
\includegraphics[width=8.2cm]{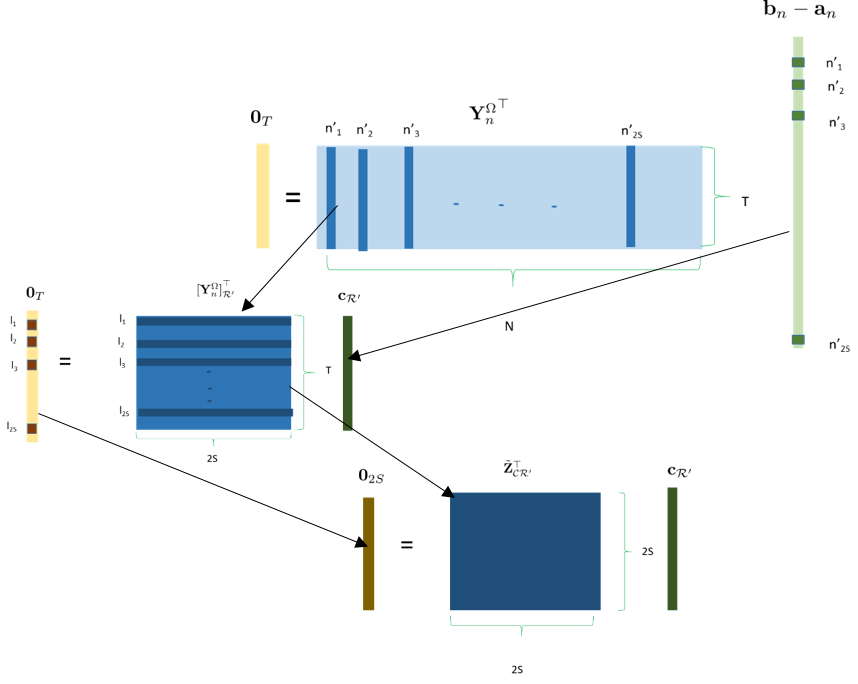}
\caption{Illustration  of the matrices in Theorem \ref{th:identsempartial} 
}\label{fig:thm2}
\end{figure}

\begin{mytheoremhere}
Under as1 and as3, the adjacency matrix $\adjacencymat$ can be uniquely 
identified from  
\eqref{eq:semvecext} and \eqref{eq:obsvecext}.
\label{th:identsempartial}
\end{mytheoremhere}
\begin{proof}
Combining  \eqref{eq:semvecext} and \eqref{eq:obsvecext} yields
\begin{align}
\label{eq:semobsvecextpar}
\observationvecwithmis\layernot{\layeridex}&=\samplematwithmis
\layernot{\layeridex} 
\adjacencymat\signalvec\layernot{\layeridex},~~~\layeridex=1,\ldots,\layernum.
\end{align}
Collecting the equations over $\timeind$ leads to  the 
$\vertexnum\times\timenum$ matrix $	\observationmatwithmis= [\samplematwithmis \layernot{1} \adjacencymat\signalvec\layernot{1},\ldots, \samplematwithmis \layernot{\layernum} \adjacencymat\signalvec\layernot{\layernum}]$.
Considering the $\vertexind$-th row of $\observationmatwithmis:= [\observationmatwithmisrow \vertexnot{1},\ldots,\observationmatwithmisrow 
\vertexnot{\vertexnum}]\transpose$ one obtains
\begin{align}
\label{eq:combinexppar}
\observationmatwithmisrow\vertexnot{\vertexind}
=&[\samplematwithmisentry{\vertexind}{\vertexind}{1}
\adjacencymatrow\vertexnot{\vertexind}\transpose\signalvec\layernot{1},\ldots,
\samplematwithmisentry{\vertexind}{\vertexind}{\layernum} 
\adjacencymatrow\vertexnot{\vertexind}\transpose\signalvec\layernot{\layernum} 
]\transpose
=(\signalmissmat\vertexnot{\vertexind})\transpose 
\adjacencymatrow\vertexnot{\vertexind}
\end{align}
%where the third equation in \eqref{eq:combinexp} follows 
%by  definition,
% of $\signalmat$ and 
%$\samplematwithmis'\vertexnot{\vertexind}$
%and 
where $\signalmissmat\vertexnot{\vertexind}=
[\samplematwithmisentry{\vertexind}{\vertexind}{1}\signalvec\layernot{1},\ldots,
\samplematwithmisentry{\vertexind}{\vertexind}{\layernum} 
\signalvec\layernot{\layernum} ].$
To argue by contradiction, suppose there exists 
a vector 
$\adjacencymatrowother\vertexnot{\vertexind}\in\rfield^{\vertexnum\times1}$ 
with $\sparsitynum$ 
nonzero entries, and 
$\adjacencymatrowother\vertexnot{\vertexind}\ne\adjacencymatrow\vertexnot{\vertexind}$,
 such 
that the 
following holds
\begin{align}
\label{eq:combinexprowdipar}
\observationmatwithmisrow\vertexnot{\vertexind}
=(\signalmissmat\vertexnot{\vertexind})
\transpose\adjacencymatrowother\vertexnot{\vertexind}.
\end{align}
Combining \eqref{eq:combinexprowdipar} with \eqref{eq:combinexppar} one arrives at
\begin{align}
\label{eq:conditionidentpar}
\mathbf{0}_\layernum
=(\signalmissmat\vertexnot{\vertexind})\transpose
(\adjacencymatrowother\vertexnot{\vertexind}-\adjacencymatrow 
\vertexnot{\vertexind}).
\end{align}
Since 
$\adjacencymatrow\vertexnot{\vertexind}$ and $
\adjacencymatrowother\vertexnot{\vertexind}$ both have $\sparsitynum$ nonzero 
entries, 
$\adjacencymatrowother\vertexnot{\vertexind}-\adjacencymatrow 
\vertexnot{\vertexind}$ has at most $2\sparsitynum$ nonzero 
entries. Without loss of generality assume that $\obsrowset'\define
\{\vertexind'_1,\vertexind'_2,\ldots,
\vertexind'_{2\sparsitynum}\}$ contains all the indices of nonzero entries in 
$\adjacencymatrowother\vertexnot{\vertexind}-\adjacencymatrow 
\vertexnot{\vertexind}$. Let  
$\mathbf{c}_{\obsrowset'}:=[\adjacencymatrowother\vertexnot{\vertexind}- 
\adjacencymatrow 
\vertexnot{\vertexind}]_{\obsrowset'}\in\rfield^{2\sparsitynum\times1}$ denote 
the sub-vector containing all entries of 
$\adjacencymatrowother\vertexnot{\vertexind}- \adjacencymatrow 
\vertexnot{\vertexind}$ indexed by $\obsrowset'$. Hence, 
\eqref{eq:conditionidentpar} can be rewritten as 
\begin{align}
\label{eq:conditionidentred1par}
\mathbf{0}_\layernum
=({[\signalmissmat\vertexnot{\vertexind}]}_{\obsrowset'})\transpose
\mathbf{c}_{\obsrowset'}
\end{align}
where ${[\signalmissmat\vertexnot{\vertexind}]}_{\obsrowset'}$ selects the 
rows of $\signalmissmat\vertexnot{\vertexind}$ indexed by $\mathcal{R}^{'}$. 
According to as3, for any 	set of row indices $\obsrowset'$ we can find a set 
of column indices such 
that $\observationmatwithmis_{\obscolumset\obsrowset'}$  is  fully observable,
and thus $\observationmatwithmis_{\obscolumset\obsrowset'}=
{[\signalmissmat\vertexnot{\vertexind}]}_{\obscolumset\obsrowset'}=
{[\signalmat]}_{\obscolumset\obsrowset'}$. As a result, it 
holds that
\begin{align}
\label{eq:conditionidentredpar}
\mathbf{0}_{2\sparsitynum}
=\observationmatwithmis_{\obscolumset\obsrowset'}\transpose
\mathbf{c}_{\obsrowset'}\;.
\end{align}
Since 
${\rm Kruskal}(\observationmatwithmis_{\obscolumset\obsrowset'}\transpose)=2\sparsitynum$, 
any  subset of 
$2\sparsitynum$  columns of 
$\observationmatwithmis_{\obscolumset\obsrowset'}\transpose$ is 
linearly independent, which implies  $	\mathbf{c}_{\obsrowset'}=\mathbf{0}$, 
and by definition $\adjacencymatrowother\vertexnot{\vertexind}=\adjacencymatrow 
\vertexnot{\vertexind}$, which leads to a contradiction.
% Notice that  as3 holds for any subset of rows, as a result  
%\eqref{eq:conditionidentred1par} holds for any 
%$\adjacencymatrowother\vertexnot{\vertexind}$. 
The analysis from \eqref{eq:combinexppar} to \eqref{eq:conditionidentredpar} 
holds for all $\vertexind=1,\ldots,\vertexnum$.
\end{proof}

%As1 effectively reduces the number of 
%unknowns to $\sparsitynum\vertexnum$ at most. As2 asserts that  the 
%observation matrix $\observationmatwithmis$ is expressive enough to identify 
%$\adjacencymatrow\vertexnot{\vertexind}$ uniquely. 
Theorem \ref{th:identsempartial} provides the sufficient conditions for identifying the network structure  given partial observations in the noise-free case. As3 implies that each node has to be sampled at least  $2\sparsitynum$ times over $\timenum$.
Fig. \ref{fig:thm2} shows the matrices involved in the proof. 
%Fig. 1 illustrates the matrices used in Theorems \ref{th:identsem} and 
%\ref{th:identsempartial}. 

\section{Numerical tests}
%\vspace{-0.1cm}
\label{sec:sims}

\cmt{topology id acessment}The tests in this section evaluate the performance of the proposed joint inference approach in comparison with state-of-the-art graph signal inference and topology identification techniques using synthetic and real data.

The network topology perfomance is measured by the edge identification  error 
rate (EIER), defined as
\begin{align*}
\text{EIER}:=\frac{\|\mathbf{S}-\hat{\mathbf{S}}\|_0}{N(N-1)}\times 100\%
\end{align*}
with the operator $\|\cdot\|_0$ denoting the number of nonzero entries of its 
argument, and $\mathbf{S}$ ($\hat{\mathbf{S}}$) the support of 
$\mathbf{A}$ ($\hat{\mathbf{A}}$). For the estimated adjacency an edge is declared 
present if $A_{n,n'}$ exceeds a threshold chosen to yield the smallest EIER.
%and the $\text{NMSE}_A=\frac{\|\mathbf{A}-\hat{\mathbf{A}}\|_F^2}{N(N-1)}$. 
The inference performance of JISG is assessed by comparing with the normalized mean-square error 
\begin{align*}
%$
\small	\text{NMSE}:=%\mathbb{E}\bigg{[}
\sum_{\layeridex=1}^ 
{\layernum} \frac{\|\signalestvec
\layernot{\layeridex} -\signalvec\layernot{\layeridex} 
\|^2_2}{
\|\signalvec\layernot{\layeridex} 
\|^2_2}.%\bigg{]}
%$
\end{align*}
Parameters $\mu$, $\lambda_1$ and $\lambda_2$ are selected via cross validation.
The software used to conduct all experiments is MATLAB. All results represent averages over 10 independent Monte Carlo runs. Unless otherwise stated, $\sampleset\timenot{\timeind}$ is chosen uniformly at random without replacement over $\vertexset$ for each $\timeind$ with constant size over time; that is, $\samplenum\timenot{\timeind}=\samplenum,~\forall \timeind$.
%		Experiments were run for different 
%		\textcolor{blue}{values of $L$}, with thresholds that control presence or absence of an 
%		edge (denoted by $\tau$ in the listed algorithms) selected to obtain the best edge 
%		identification accuracy. Furthermore, sparsity-promoting regularization parameters 
%		($\lambda$) were all judiciously selected to obtain the lowest edge identification  error 
%		rate (EIER), defined as
%			%
%			$
%			\text{EIER}:=\frac{\|\mathbf{A}-\hat{\mathbf{A}}\|_0}{N(N-1)}\times 100\%
%			$,
%			%
%			with the operator $\|\cdot\|_0$ denoting the number of nonzero entries of its 
%			argument. For all experiments, error plots were generated using values of EIER 
%			averaged over $100$ independent runs. 
%	\cmt{reconstruction perfomance accessment}
%	    		We assess the graph signal inference performance via Monte Carlo
%	    		simulation by comparing the normalized mean-square error (NMSE),
%	    		%\begin{align}
%	    		$\text{NMSE}=\mathbb{E}\bigg{[}\sum_{\layeridex=1}^ 
%	    		{\layernum}\|\signalestvec
%	    			\layernot{\layeridex} -\signalvec\layernot{\layeridex} 
%	    			\|^2/\|\signalvec\layernot{\layeridex} 
%	    			\|^2\bigg{]}.$
%\end{align}
%	    		averaged over choices of the sample set $\sampleset$ and, for 
%synthetic
%	    		data experiments, also over noise and signal realizations.
%	    		\cmt{monte carlo runs, matlab,etc}

\subsection{Numerical tests on synthetic data}

\cmt{Graph description}First, a synthetic network of size 
$\vertexnum=81$ 
was generated using the \emph{Kronecker product} model, that effectively captures properties 
of 
real graphs~\cite{leskovec2010kronecker}. It relies on the ``seed matrix'' 
\begin{align*}
\small
\seedzeromat:=\left[ 
\begin{array}{ccc}
0.6   & 0.1 & 0.7\\
0.3  & 0.1 & 0.5 \\
0  & 1 & 0.1
\end{array}
\right]
\end{align*}  		
%	\vspace*{-0.1cm}
that produces the  $\vertexnum\times\vertexnum$ matrix as{
$\small\seedmat\!:=\!\seedzeromat
\!\otimes\!\seedzeromat\!\otimes\!\seedzeromat\!\otimes\! 
\seedzeromat$}, where $\otimes$ denotes Kronecker product.
The entries of $\adjacencymat$ were selected as
$\adjacencymatentry\vertexvertexnot{\vertexind}{\vertexindp}\sim\text{Bernoulli}
(\seedmatentry\vertexvertexnot{\vertexind}{\vertexindp})~\forall \vertexind,\vertexindp$, and 
the 
resulting matrix was rendered symmetric by adding its transpose. The graph signals were generated using the graph-bandlimited model
$
\signalvec\layernot{\layeridex}=\sum_{i=1}^{10}\gamma\layernot{\layeridex}^{(i)} 	  
\laplacianevec^{(i)},~\layeridex=1,\ldots,\layernum
$,
where $\layernum=100$, 
$\gamma\layernot{\layeridex}^{(i)}\!\sim\!\mathcal{N}(0,1)$, and  $\{
\laplacianevec^{(i)}\}_{i=1}^{10}$ are the 
eigenvectors
associated with the 10 smallest eigenvalues of the Laplacian matrix $\laplacianmat :=
\text{diag}\{\adjacencymat\boldsymbol 1\}-\adjacencymat$. No noise was added to the observations; that is, $\observationnoisevec\timenot{\layeridex}=\mathbf{0},~\forall\layeridex$.
% The resulting 
%functions~\eqref{eq:bandlimited} are smooth over the graph. 

\cmt{Experiment 2:NMSE vs time: graph function infer
\ra Fig.~\ref{fig:nmsedifmodel}}The compared estimators for graph signal 
inference include the 
bandlimited estimator    
(BL)~\cite{narang2013localized,  
anis2016proxies} %
% which assumes the 
% signal follows~\eqref{eq:bandlimited}
with bandwidth $\bandwidth$;  and the 
multi-kernel learning (MKL) estimator that employs a dictionary comprising 100 diffusion 
kernels 
with 
parameter 
$\sigma^2$ uniformly spaced between 0.01 and 2, and selects the 
kernel that ``fits" best the 
observed
data~\cite{romero2016multikernel}. These reconstruction algorithms assume the topology is 
known and symmetric, 
which may not always be the case. To capture 
model mismatch, BL and MKL use 
$\adjacencymat+\noiseadjacencymat$ with 
$\noiseadjacencymatentry\vertexvertexnot{\vertexind}{\vertexindp}\sim\mathcal{N}(0,0.05)$ 
instead of $\adjacencymat$.
\cmt{explain fig}
Fig.~\ref{fig:nmsedifmodel} shows the NMSE  of various approaches with increasing  
$\samplenum$, where $\samplenum\layernot{\layeridex}=\samplenum~\forall\layeridex$, 
and the baseline is the 
MKL that considers the true topology 
$\adjacencymat$. \cmt{Punchline} 
The reconstruction performance of JISG is superior compared to that of BL and MKL, and matches the baseline performance. Moreover,  the reported  CPU time  of JISG at 0.12 seconds was an order of magnitude faster than that of the MKL baseline  at 1.6 seconds.

\begin{figure}[t]
\hspace{-0.5cm}
\centering\includegraphics[width=9.2cm]{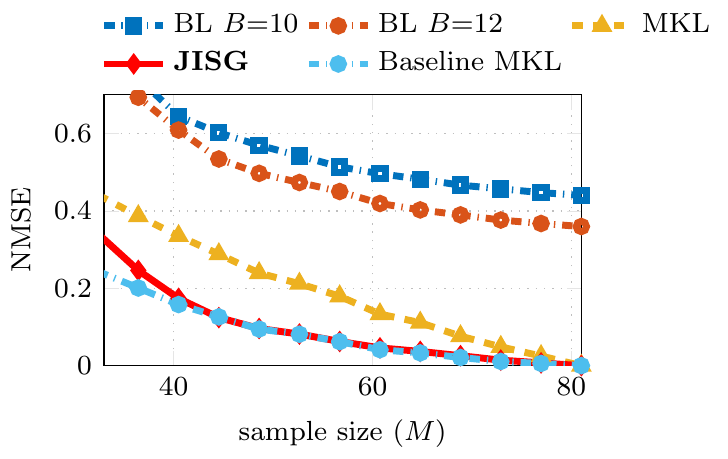}
\caption{Graph signal inference performance based on NMSE 
($\mu=10^4, \lambda_1= 0.5, \lambda_2=0.1$).} 
\label{fig:nmsedifmodel}
\end{figure}

\cmt{compare with topology id method that observes noisy signals EN-SEM}
For the same simulation setting, %and $\samplenum=$
the topology inference 
performance was evaluated, by comparing  with the 
elastic net (EN) SEM that identifies the network topology from 
observations across all nodes, meaning $\{\observationvec\layernot{\layeridex}= 
\signalvec\layernot{\layeridex}\}_{\layeridex=1}^{\layernum}$.
%EN-SEM uses the same parameters ($\lambda_1,\lambda_2$)  as JISG.
Fig.~\ref{fig:eier} plots the EIER  with increasing  $\samplenum$ for 
JISG while EN-SEM uses 
$\samplenum=\vertexnum$.  The semi-blind novel approach achieves similar performance with the 
baseline, which can not cope with missing nodal measurements.

\begin{figure}[t]
\centering{\includegraphics[width=8.6cm]{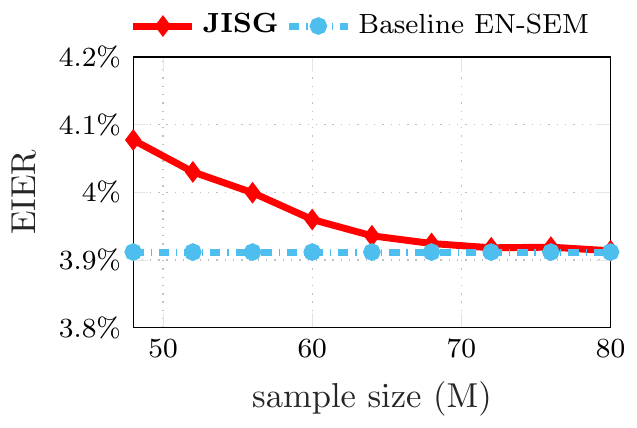}}
%{\includegraphics[width=0.47\linewidth]{misc/KFonGSimulations_3219-1}}%\vspace{-1em}
%\hfill
%	\vspace{-0.7cm}
\caption{Network topology inference performance based on EIER. 
EN-SEM uses 
$\samplenum=\vertexnum$ and the same parameters 
$\lambda_1,\lambda_2$ as JISG.
($\mu=10^4, \lambda_1= 100, \lambda_2=1$).} 
% 
%temperature estimates. 
%	($\regparone=1$, 
%		$\regpartwo=1$, $\dlsrstepsize
%		=1.6$, $\dlsrbeta=0.5$, $\lmsstepsize =0.6$, 
%		$\transweight=10^{-3}$, 
%		$\gausmeanstatenoise=10^{-5}$, 
%$\gausstdstatenoise=10^{-6}$, 
%		$\gausmeanspatio=2$, 
%		$\gausstdspatio=0.5$, $\rkhsspationum=40$, 
%$\rkhsstatenoisenum=40$)}
%\label{fig:recon}
\label{fig:eier}
%\vspace{-1em}
\end{figure}
\begin{figure}
%\vspace{-0.4cm}
\begin{subfigure}[h]{0.45\linewidth}

%\centering{\input{trueadj.tex}}
\centering\includegraphics[width=1\linewidth]{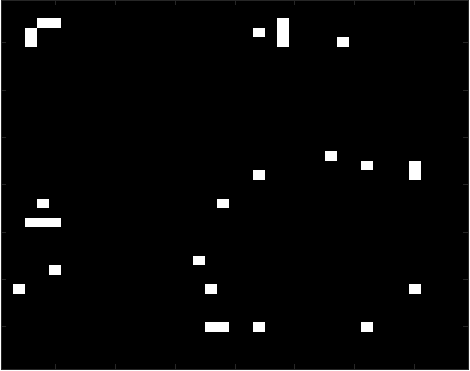}
\caption{EN-SEM.}
\label{fig:true}
\end{subfigure}%
\hspace{3mm}
\begin{subfigure}[h]{0.5\linewidth}
%\centering{\input{JISGest.eps}}
\centering\includegraphics[width=1\linewidth]{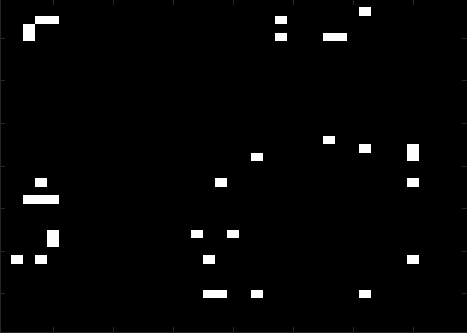}
\caption{{JISG}.}
\label{fig:est}
\end{subfigure}
%	\vspace{-0.1cm}
\caption{Heatmaps of estimated adjacency matrices for the gene regulatory 
network.  
White (black) 
indicate  presence
(absence) of an edge. ($\mu=10^4, \lambda_1= 10^{-2}, \lambda_2=10^{-4}$).}
\label{fig:heat}
\end{figure}
\cmt{measure NMSE and edge prediction rate check yanning code}

\subsection{Gene regulatory network identification}
%    g{The experiment was based on real gene expression data resulting 
%    from RNA sequencing of cell samples from $69$ unrelated Nigerian individuals, under the 
%    International HapMap project~\cite{frazer2007second}. From the $929$ identified genes, 
%    expression levels and the genotypes of the \emph{expression quantitative trait loci (eQTLs)} 
%    of $39$ immune-related genes were selected and normalized; 
%    see~\cite{cai2010memoryless} and ~\cite{pickrell2010understanding} for detailed 
%    descriptions. 
%    	Gene expression levels were treated as the endogenous variables $\mathbf{Y}$. 
%    	Parameters $\mu$, $\lambda_1$ and $\lambda_2$ are selected via cross validation.}
Further tests were conducted using real gene exrpession data~\cite{cai2013inference}.
\cmt{Graph description}Nodes in this network represent 
$\vertexnum=39$ 
immune-related
genes, while the measurements consist of gene expression
data from $\layernum=69$ unrelated Nigerian individuals.
\cmt{Signal description}The graph process 
$\signalfun\vertlayernot{\vertexind}{\layeridex}$ measures the expression level of gene 
$\vertexind$ for individual $\layeridex$. 
\cmt{Experiment}This experiment evaluates the topology inference performance of JISG 
with $\samplenum\layernot{\layeridex}=31$ genes for all individuals 
sampled at random. 
Since 
no ground-truth topology is available here, the estimated adjacency of 
EN-SEM, that relies on
all the  observations, was used for comparison. 
\cmt{explain fig}Fig.~\ref{fig:heat} depicts heatmaps of the 
estimated adjacencies.
\cmt{Punchline}As observed, JISG  learns 
a  topology similar to that identified by  EN-SEM, and imputes the missing values with
NMSE~$=0.017$. Therefore, our joint inference approach is capable of revealing causal dependencies even when gene expression data contain missing values.
%, a common case due to insufficient resolution, 
%image corruption, 
%etc; see\cite{kim2004missing}.		

%\subsection{Tests on JISGoT}

%\subsection{Recommender systems}

\subsection{Temperature prediction}

Consider the National Climatic dataset, which comprises  hourly temperature measurements at $\vertexnum=109$ measuring stations across the continental United States in 2010 \cite{USATemp}. The value $\signalfun\vertimenot{\vertexind}{\timeind}$  represents here the $\timeind$-th temperature sample recorded at the $\vertexind$-th  station. For evaluating the JISGoT the cumulative NMSE (cNMSE) was used
\begin{align*}
\text{cNMSE}(\timenum):=\frac{\sum_{\timeind=1}^{\timenum}
||\signalvec\timenot{\timeind}
-\signalestvec\timenot{\timeind|\timeind}||^2_2}
{\sum_{\timeind=1}^{\timenum}	 
||\signalvec
\timenot{\timeind}||^2_2}.
\end{align*}
\begin{figure}[t]
\hspace{-1cm}\centering{\includegraphics[width=8.2cm]{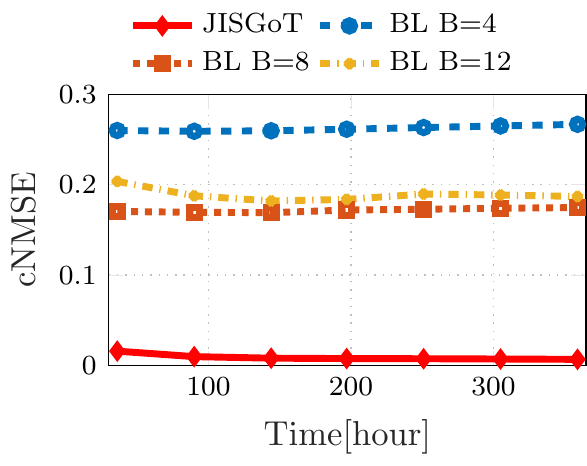}}
\caption{NMSE for temperature estimates
($\mu=100, \lambda_1= 10, \lambda_2=10^4$).} 
\label{fig:nmsetempbandot}
\end{figure}
Next, the proposed method is compared to the graph-bandlimited approach \cite{anis2016proxies,narang2013localized} for different bandwidth values $\bandwidth$, where  a time-invariant graph was constructed  as in 
\cite{romero2016spacetimekernel}, based on geographical distances.
Fig. \ref{fig:nmsetempbandot} reports the cNMSE performance of the estimators with $\samplenum=76$ for increasing $\timenum$. JISGoT learns the latent topology among sensors and outperforms the band-limited estimator since the temperature may not adhere to the band-limited model for the geographical graph.

\begin{figure}[t]
\hspace{-1cm}\centering{\includegraphics[width=8.2cm]{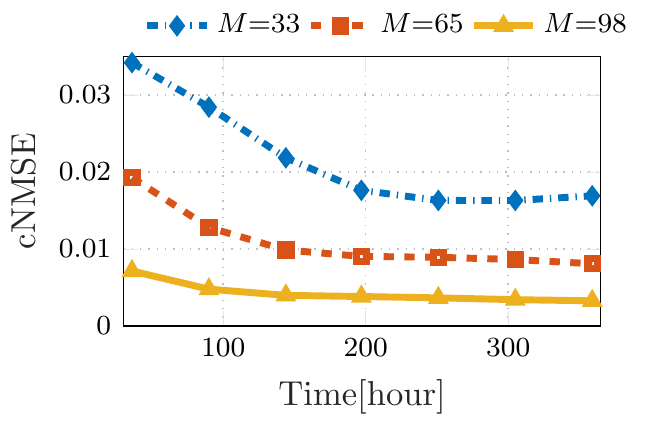}}
\caption{NMSE of JISG estimates for different $\samplenum$ 
($\mu=100, \lambda_1= 10, \lambda_2=10^4$).} 
\label{fig:nmsetempotindifS}
\end{figure}
Fig. \ref{fig:nmsetempotindifS} shows the cNMSE of JISGoT with variable $\samplenum$. As expected, the performance improves with increasing number of samples, while with just 30\% sampled stations the normalized reconstruction error is only 0.018. Hence, JISGoT can be employed to effectively predict missing sensor measurements.
\subsection{GDP prediction}
This experiment is carried over the gross domestic product (GDP) dataset~\cite{wordbank}, which comprises GDP per capita  for $\vertexnum=127$ countries  for the years  1960-2016. The process $\signalfun\vertimenot {\vertexind}{\timeind}$ now denotes the GDP  reported at the $\vertexind$-th  country and $\timeind$-th year for $\timeind=1960,\ldots,2016$. Fig. \ref{fig:nmsegdpotindifS} shows the cNMSE performance of our joint approach for different $\samplenum$. The semi-blind estimator unveils the latent connections among countries, while it reconstructs the GDP with cNMSE=0.05 when  60\% samples are available.
\begin{figure}[t]
\centering{\includegraphics[width=9cm]{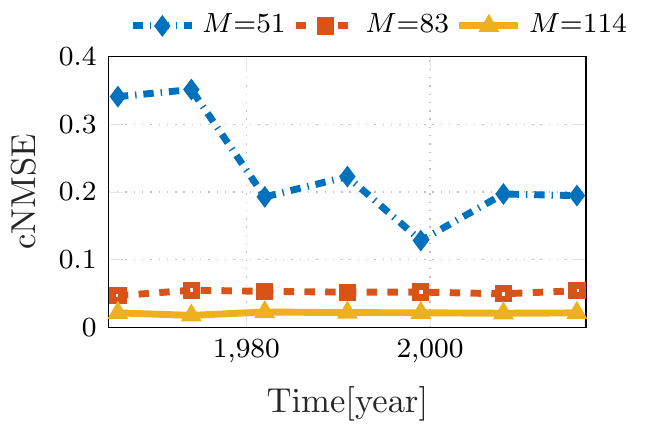}}
\caption{NMSE for GDP estimates 
($\mu=100, \lambda_1= 0.01, \lambda_2=1$).} 
\label{fig:nmsegdpotindifS}
\end{figure}
\begin{figure}[t]
\centering{\includegraphics[width=8.2cm]{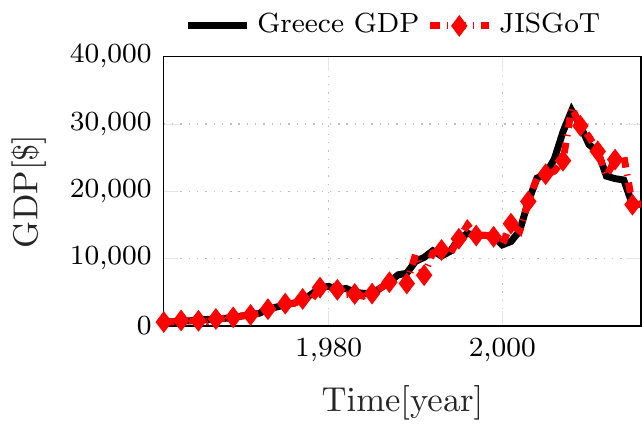}}
\caption{Greece GDP along with  JISG  estimate
($\mu=100, \lambda_1= 0.01, \lambda_2=1$).} 
\label{fig:nmsegdpottrack}
\end{figure}

Fig. \ref{fig:nmsegdpottrack} depicts the true values, along with the GDP estimates of Greece for $\samplenum=89$, which corroborates the effectiveness of JISGoT in predicting the GDP evolution and henceforth facilitating economic policy planning.

\subsection{Network delay prediction}
\cmt{Network Delay dataset}

\cmt{Description}
The  last dataset records  measurements of path delays on the Internet2 
backbone\cite{internet2}. The network comprises 9 end-nodes and $26$ directed 
links. The delays are available for $\vertexnum=70$ paths per minute. 
Function $\signalfun\vertimenot{\vertexind}{\timeind}$ denotes the delay in milliseconds measured at the $\vertexind$-th  path and
$\timeind$-th minute. 

%Next, the proposed joint inference algorithm is compared to the MKriKF approch in~\cite{ioannidis2017kriged} that estimates graph functions using a kernel dictionary created from an assumed topology based on the routing matrix.  Fig. \ref{fig:nmsenetdelpotindifS} reports the cNMSE of the estimators with $\samplenum=40$, where MKriKF is configured as in~\cite{ioannidis2017kriged}. Clearly, JISGoT outperforms the competing approach, which may be attributed to deficient topology selection.

The proposed JISGoT will be evaluated in estimating
the delay over the 
network from $\samplenum=49$ randomly sampled path delays. 
To that end, delay maps are traditionally employed, which depict the network 
delay 
per path over time 
and enable  operators to perform 
troubleshooting. The paths for the 
delay maps in Fig. 11
are sorted in  increasing 
order  of the true delay at $\timeind=1$. Clearly, the 
delay map recovered by  JISGoT in
Fig.~\ref{fig:delJsig} visually resembles the true delay 
map in Fig.~\ref{fig:deltrue}. 
%\begin{figure}[t]
%	\centering{\input{41113.tex}}
%	\caption{Graph signal inference performance based on cNMSE 
%		($\samplenum$=40,$\mu=10^4, \lambda_1= 0.5, \lambda_2=0.1$).} 
%	\label{fig:nmsenetdelpotindifS}
%\end{figure}

\begin{figure*}[t!]\label{fig:deltot}
\centering
\begin{subfigure}[t]{0.5\textwidth}
\centering{\includegraphics[width=1\linewidth]{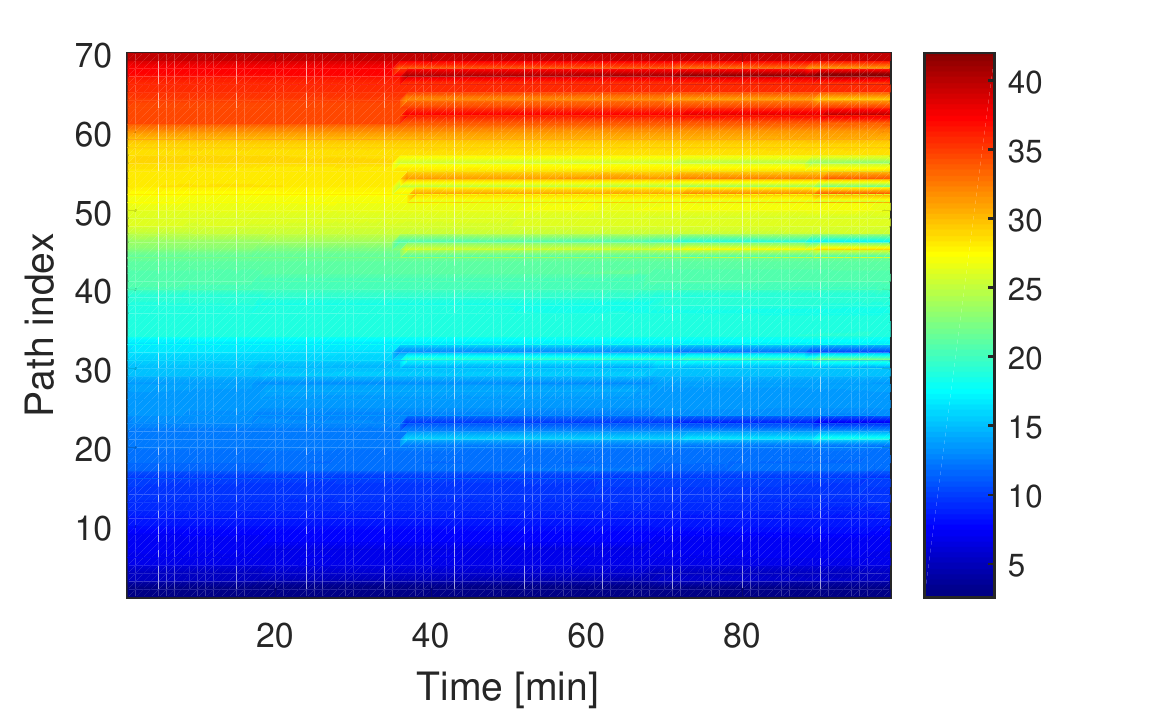}}
\caption{True delay} 
\label{fig:deltrue}
\end{subfigure}%
~ 
\begin{subfigure}[t]{0.5\textwidth}
\centering{\includegraphics[width=1\linewidth]{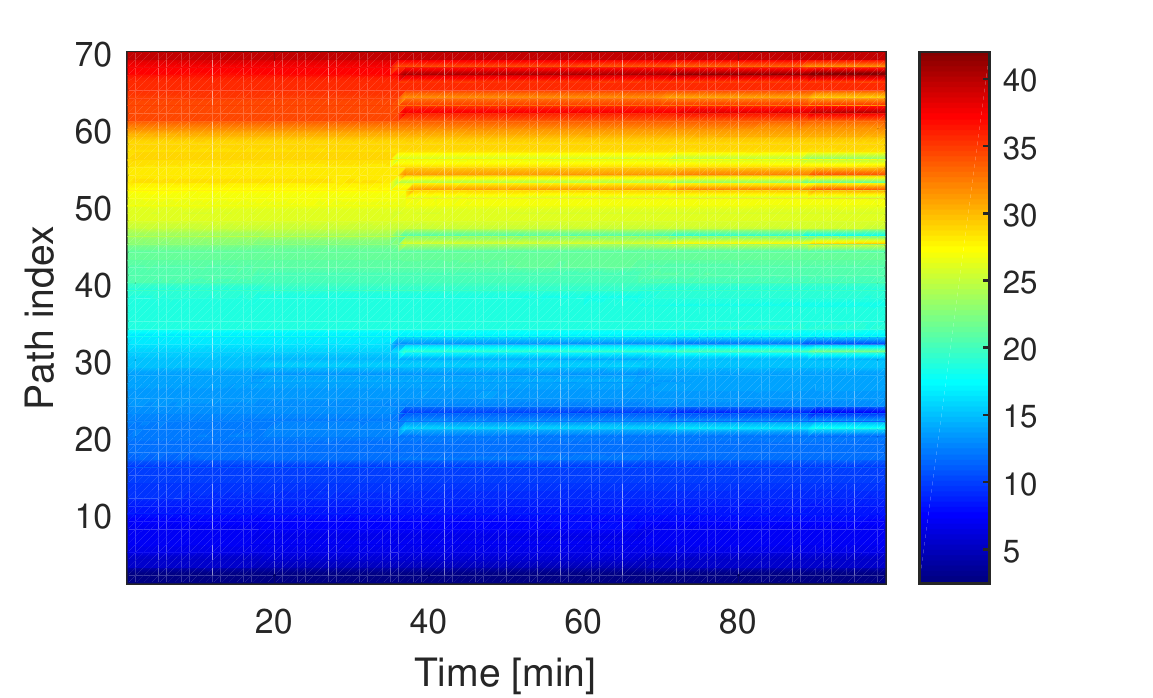}}
\caption{JISGoT} 
\label{fig:delJsig}
\end{subfigure}
\caption{True and estimated network delay map for $\vertexnum=70$ paths 
($\mu=100, \lambda_1= 0.01, \lambda_2=0.01$)}
\end{figure*}

%\begin{figure}
%	\def\tabularxcolumn#1{m{#1}}
%	\begin{tabularx}{\linewidth}{@{}cX@{}}
%		\begin{tabular}{c}
%			\subfloat[True delay]{
%				\label{fig:deltrue}
%				\hspace{0.9cm}\input{40002true.tex}
%			}
%			\\	\subfloat[JISGoT]{
%				\label{fig:delmkrikf}
%				\hspace{-0.3cm}\input{40002jsig.tex}
%			} 
%		\end{tabular}
%	\end{tabularx}
%	\caption{True and estimated network delay map for $\vertexnum=70$ paths
%		($\samplenum=49$).}
%	\label{fig:trackdel}
%\end{figure}

\section{Conclusions and future work}
%\vspace{-0.5em}
\label{sec:concl}
This paper puts forth a novel framework based on SVARMs and SEMs to jointly infer sparse directed network topologies, and even dynamic graph processes. Efficient  minimization approaches are developed with provable convergence that alternate between reconstructing the network processes and inferring the topologies using ADMM.  The framework was broadened  to facilitate real-time sequential joint-estimation by employing a fixed-lag solver. Recognizing the challenges related to  partially  observed processes, conditions under which the network can be uniquely identified were derived. Numerical tests on synthetic and real data-sets demonstrate the competitive performance of JISG and JISGoT in both inferring graph signals and the underlying network topologies. 

Future research could pursue learning nonlinear models of the network processes, and distributed implementation  of JISG, which is well-motivated, especially when dealing with large-scale networks.
\appendix
\subsection{ADMM solver for \eqref{eq:topinfersvarm}} 
Towards deriving the ADMM solver, consider 
$\signalmat_{\timenum}$ $:=[\signalvec\timenot{1},\signalvec\timenot{2},\ldots\signalvec 
\timenot{\timenum}]\transpose$, $\signalmat_{\timenum-1}$ 
$:=[\signalvec\timenot{0},\signalvec\timenot{1},\ldots\signalvec 
\timenot{\timenum-1}]\transpose$,  and the auxiliary variables  $\adjacencymathelp\lagnot{0}$ 
and $\adjacencymathelp\lagnot{1}$. Then, re-write \eqref{eq:topinfersvarm} as
\begin{align}
\label{eq:topinfersvarmrefor}
\hspace{-1cm}\underset{\adjacencymat\lagnot{0}, \adjacencymathelp\lagnot{0},
\atop\adjacencymat\lagnot{1},\adjacencymathelp\lagnot{1}
}{\min}&
%{\lossfunction(\adjacencymat,\{\signalvec\layernot\layeridex\}_{\layeridex=1}
%	^\layernum ):=} 
\frac{1}{2}\|\signalmat_{\timenum}\transpose
-\adjacencymat\lagnot{0} \signalmat_{\timenum}\transpose
-\adjacencymat\lagnot{1}\signalmat_{\timenum-1}\transpose
\|_F^2\nonumber+\regparadjone\lagnot{0}\|\adjacencymathelp\lagnot{0}\|_1\\
%+\regparadj\regfun(\adjacencymat)
+&\frac{\regparadjtwo\lagnot{0}}{2}\| 
\adjacencymat\lagnot{0}\|^2_F+\regparadjone\lagnot{1}\|\adjacencymathelp\lagnot{1}\|_1
+\frac{\regparadjtwo\lagnot{1}}{2}\| \adjacencymat\lagnot{1}\|^2_F\nonumber\\
\st~~& 
\adjacencymat\lagnot{0}=\adjacencymathelp\lagnot{0}-
\diag{\adjacencymathelp\lagnot{0}}, 
\adjacencymat\lagnot{1}=\adjacencymathelp\lagnot{1}.
\end{align}
The augmented Lagrangian of~\eqref{eq:topinfersvarmrefor} is 
\begin{align}
\label{eq:augmentLagtoinfersvar}
\hspace{-1cm}\mathcal{L}&=
%{\lossfunction(\adjacencymat,\{\signalvec\layernot\layeridex\}_{\layeridex=1}
%	^\layernum ):=} 
\frac{1}{2}\|\signalmat_{\timenum}\transpose
-\adjacencymat\lagnot{0} \signalmat_{\timenum}\transpose
-\adjacencymat\lagnot{1}\signalmat_{\timenum-1}\transpose
\|_F^2\nonumber+\regparadjone\lagnot{0}\|\adjacencymathelp\lagnot{0}\|_1\\
%+\regparadj\regfun(\adjacencymat)
+&\frac{\regparadjtwo\lagnot{0}}{2}\| 
\adjacencymat\lagnot{0}\|^2_F+\regparadjone\lagnot{1}\|\adjacencymathelp\lagnot{1}\|_1
+\frac{\regparadjtwo\lagnot{1}}{2}\| \adjacencymat\lagnot{1}\|^2_F\nonumber\\
+&\tr{{\admmhelp\lagnot{0}}\transpose(\adjacencymat\lagnot{0}- 
\adjacencymathelp\lagnot{0}+ 
\diag{\adjacencymathelp\lagnot{0}})}\nonumber
\\+
&\frac{\admmreg}{2}\|\adjacencymat\lagnot{0}- 
\adjacencymathelp\lagnot{0}+ 
\diag{\adjacencymathelp\lagnot{0}}\|_F^2\nonumber
\\+&\tr{{\admmhelp 
		\lagnot{1}}\transpose(\adjacencymat\lagnot{1}- 
	\adjacencymathelp\lagnot{1})}+\frac{\admmreg}{2}
\|\adjacencymat\lagnot{1}- 
\adjacencymathelp\lagnot{1}\|_F^2
\end{align}
where $\admmhelp\lagnot{0}$ and $\admmhelp\lagnot{1}$ denote Lagrange multiplier matrices, while 
$\admmreg>0$ is the penalty parameter. Henceforth, square brackets denote ADMM iteration 
indices.  The ADMM update for $\adjacencymat\lagnot{0}$ results from
${\partial\mathcal{L}}/{\partial\adjacencymat\lagnot{0}}=\bf 0$ that gives
\begin{align}
\label{eq:ad0update}
\adjacencymat\lagnot{0}
	\admmnot{\admmiter}
	(\cormat\timenot{T}&+(\regparadjtwo\lagnot{0}+\admmreg)\identitymat_\vertexnum) =
	\\\cormat\timenot{T}-&\adjacencymat\lagnot{1} 
	\admmnot{\admmiter-1}\crosscormat\transpose
	-\admmhelp\lagnot{0}\admmnot{\admmiter-1}+ 
	\admmreg\adjacencymathelp\lagnot{0}\admmnot{\admmiter-1}\nonumber
\end{align}
where $\cormat\timenot{T}:=\signalmat\timenot{T}\transpose\signalmat\timenot{T}$,  
$\crosscormat:=\signalmat\timenot{T}\transpose\signalmat\timenot{T-1}$. Similarly for 
$\adjacencymat\lagnot{1}$,  taking
${\partial\mathcal{L}}/{\partial\adjacencymat\lagnot{1}}= \bf 0$ results to
\begin{align}
\label{eq:ad1update}
\adjacencymat
 \lagnot{1}
\admmnot{\admmiter}
(\cormat\timenot{T-1}&+(\regparadjtwo\lagnot{1}+\admmreg) 
\identitymat_\vertexnum 
)=\nonumber\\\crosscormat-&
\adjacencymat\lagnot{0} 
\admmnot{\admmiter} \crosscormat
-\admmhelp\lagnot{1}\admmnot{\admmiter-1}+ 
\admmreg\adjacencymathelp\lagnot{1}\admmnot{\admmiter-1}
\end{align}
where $\cormat\timenot{T-1}:=\signalmat\timenot{T-1}\transpose\signalmat\timenot{T-1}$.  
The elementwise soft-thresholding operator is defined as
  \begin{align*}
  \mathcal{T}_\alpha(x)=\left\{
  \begin{array}{ll}
  x-\alpha, ~~x>\alpha\\
 0, ~~~~~~~|x|\le\alpha\\
  x+\alpha, ~~x<-\alpha.
  \end{array}
  \right.
  \end{align*}
  Accordingly, the update for $\adjacencymathelp\lagnot{0}$ is 
  \begin{align}
  \label{eq:adhelp0update}
  \adjacencymathelp\lagnot{0}\admmnot{\admmiter}&=\mathcal{T}_{
  	\regparadjone\lagnot{0}/\admmreg}
  (\adjacencymat\lagnot{0} 
  \admmnot{\admmiter}+\frac{1}{\admmreg}\admmhelp\lagnot{0}\admmnot{\admmiter-1}) 
  \nonumber\\
 -&\diag{\mathcal{T}_{\regparadjone\lagnot{0}/\admmreg}
 	(\adjacencymat\lagnot{0} 
 	\admmnot{\admmiter}+\frac{1}{\admmreg}\admmhelp\lagnot{0}\admmnot{\admmiter-1})}
  \end{align}
  and for  $\adjacencymathelp\lagnot{1}$ it is 
    \begin{align}
    \label{eq:adhelp1update}
    \adjacencymathelp\lagnot{1}\admmnot{\admmiter}&=\mathcal{T}_{\regparadjone
    	\lagnot{1}/\admmreg}
    (\adjacencymat\lagnot{1} 
    \admmnot{\admmiter}+\frac{1}{\admmreg}\admmhelp\lagnot{1}\admmnot{\admmiter-1}).
    \end{align}
    Finally, the Lagrange multiplier updates are given by   
    \begin{subequations}
    	\label{eq:helpupdate}
    	 \begin{align}
    \admmhelp\lagnot{0}\admmnot{\admmiter}=\admmhelp\lagnot{0}\admmnot{\admmiter-1} 
    +\admmreg(\adjacencymat\lagnot{0}\admmnot{\admmiter} -\adjacencymathelp\lagnot{0} \admmnot{\admmiter})\\
    \admmhelp\lagnot{1}\admmnot{\admmiter}=\admmhelp\lagnot{1}\admmnot{\admmiter-1} 
    +\admmreg(\adjacencymat\lagnot{1}\admmnot{\admmiter} -\adjacencymathelp\lagnot{1} \admmnot{\admmiter}).
     \end{align}
    \end{subequations}
    The complexity of \eqref{eq:ad0update} and \eqref{eq:ad1update} is 
    $\mathcal{O}(\vertexnum^3)$, while for \eqref{eq:adhelp0update}-\eqref{eq:helpupdate} is 
    $\mathcal{O}(\vertexnum^2)$ that leads to an overall per ADMM iteration complexity of 
    $\mathcal{O}(\vertexnum^3)$, governed by the updates for
    $\adjacencymat\lagnot{0}$ 
    and $\adjacencymat\lagnot{1}$. That brings the overall complexity of the algorithm to 
   $\mathcal{O}(I\vertexnum^3)$, where $I$ is the number of required ADMM iterations until convergence.  

\small
\clearpage
\bibliographystyle{IEEEtranS}
\bibliography{my_bibliography,net}
\noindent
\end{document}